\documentclass[11pt]{article}
\usepackage[utf8]{inputenc}
\usepackage{amssymb,amsmath,amsthm,graphicx,geometry,hyperref,xcolor,bm,mathtools,authblk,subcaption,multirow,tensor,mathabx,tikz, algorithm,algpseudocode,subcaption}
\usepackage[numbers]{natbib}

\newgeometry{left=1in,right=1in,top=1in,bottom=1in}

\hypersetup{
    colorlinks=true,
    linkcolor=blue,
    filecolor = blue
, urlcolor = blue,
citecolor = blue
}

\newtheorem{theorem}{Theorem}

\newtheorem{corollary}{Corollary}
\newtheorem{lemma}{Lemma}
\newtheorem{definition}{Definition}

\newcommand{\tr}{\text{tr}}
\newcommand{\bl}[1]{{\mathbf #1}}
\newcommand{\bs}[1]{\boldsymbol #1}
\newcommand{\mb}[1]{\mathbb #1}
\newcommand{\mc}[1]{\mathcal #1}

\newcommand{\argmin}{\text{argmin}}
\newcommand{\argmax}{\text{argmax}}

\begin{document}

\title{The Nondecreasing Rank}
\author{Andrew McCormack} 
\affil{Department of Mathematical and Statistical Sciences, University of Alberta} 
\date{\today}

\maketitle

\begin{abstract}
In this article the notion of the nondecreasing (ND) rank of a matrix or tensor is introduced.  A tensor has an ND rank of $r$ if it can be represented as a sum of $r$ outer products of vectors, with each vector satisfying a monotonicity constraint. It is shown that for certain poset orderings finding an ND factorization of rank $r$ is equivalent to finding a nonnegative rank-$r$ factorization of a transformed tensor. However, not every tensor that is monotonic has a finite ND rank. Theory is developed describing the properties of the ND rank, including typical, maximum, and border ND ranks. Highlighted also are the special settings where a matrix or tensor has an ND rank of one or two. As a means of finding low ND rank approximations to a data tensor we introduce a variant of the hierarchical alternating least squares algorithm. Low ND rank factorizations are found and interpreted for two datasets concerning the weight of pigs and a mental health survey during the COVID-19 pandemic.

\end{abstract}

\section{Introduction}
Matrix and tensor factorizations are indispensable tools in data science, both for finding easily interpretable representations of complex data, and for obtaining more stable estimates of estimands of interest. An important example of a matrix factorization technique is nonnegative matrix factorization (NMF) \citep{GillisNNMF}, which has been leveraged in many applications ranging from hyperspectral imaging \citep{HyperspectralUnmixing} to recommender systems \citep{RecommenderSystemNMF}. In a nonnegative matrix factorization a matrix is decomposed into a product of two, low-rank, nonnegative matrices. The rows and columns of the matrix factors often have intuitive interpretations as atomic units that the dataset is built from. 

Nonnegative matrix factorizations have a statistical motivation as they also arise in the study of contingency tables as mixture models of independent random variables
\citep[Ch 4]{DrtonLecturesonAlgstats}. One strand of literature within the field of algebraic statistics has been to understand the semialgebraic geometry of such mixture models. Despite being simple to formulate, the geometry of NMF is involved, even in the rank-$3$ case \citep{KubjasNMFRank3Geometry, KubjasRobevaExpMaxNMF}. 

The present paper is concerned with a generalization of NMF; in addition to nonnegativity constraints, we seek matrix and tensor factorizations that fulfill monotonicity constraints with respect to a user-specified partial order. Termed a nondecreasing (ND) factorization, much of the theory for nonnegative factorizations will hold in this generalized setting. There are some notable exceptions where the theory for ND factorizations differs. For instance, the maximum ND rank of a matrix can be significantly larger than the maximum nonnegative rank. 

Monotonicity constraints with respect to a partially ordered set feature prominently in order constrained statistical inference \citep{ConstrainedStatInferenceSilvapuleSen}. 
Order constraints often arise in hypothesis testing problems where the efficacy of a collection of treatments are compared to determine which treatment appears to produce the best results \citep{TreamentOrderConstrainedsnapinn1987evaluating}. When observing functional data order constraints are also relevant. It may for example be natural to require that a function be monotone increasing \citep{RamsayMonotone}. Monotonicity constraints for matrix or tensor data have been previously modeled by multi-way, order constrained, analyses of variance (ANOVA) \citep{OrderedANOVA}. ND factorization, while similar, is not equivalent to order constrained ANOVA. 

An overview of this paper is as follows: Section \ref{sec:background} provides background material and Section \ref{sec:NonDecRankIntro} provides further introduction and motivation for the concept of the ND rank. Sections \ref{sec:OrderConeGeomExisenceofFactor} and \ref{sec:EquivofNDandNNRanks} address the questions of when exact ND factorizations exist and when they can be reframed as nonnegative factorizations up to a linear change in coordinates. The geometry of order cones is also explicated here. Section \ref{sec:MaxTypicalRanks} provides exact expressions and bounds for the maximum ND rank in certain cases. It is shown that knowing the maximum ND rank is enough to fully determine the typical ND ranks. Section \ref{sec:NDRankOneTwo} shows that if the rank of a matrix is one or two and it has a finite ND rank, then the ND rank is also equal to one or two respectively. Section \ref{sec:NDBorderRank} shows that the border ND rank is equivalent to the ND rank, a fact which is relevant for optimization. Section \ref{sec:Optimization} proposes an ND hierarchical least squares optimization algorithm. The paper concludes with an application section illustrating how ND factorizations are fruitful for uncovering structure in data.

\section{Notation}
\label{sec:Notation}
Some of the notation used in this work is listed here for easy reference.  Partially ordered sets (posets) are represented by $\mc{P}$, the order cone associated with this poset is denoted by $\mc{C}(\mc{P})$, and the set of tensors with ND rank at most $r$ is $\mc{N}_{\leq r}$. The dual of a cone $\mc{C}$ is $\mc{C}^*$.   
 A curly inequality $\prec$ represents inequality with respect to a poset while $x \lessdot y$, means that the element $y$ covers $x$ in the poset. $\mb{R}^p_+$ is equal to all vectors in $\mb{R}^p$ with nonnegative entries and $\bl{e}_i$ is the $i$th standard basis vector. Vectors, matrices, and tensors are all displayed in bold font, while scalars are not. Given any finite-dimensional vector spaces $V_1,\ldots,V_k$ and sets $S_i \subseteq V_i$,  $\otimes_{i = 1}^k V_i$ is the tensor product space of the $V_i$ and $\otimes_{i = 1}^k S_i$ is the subset $\cup_{r = 1}^\infty \{\sum_{i = 1}^r \otimes_{j = 1}^k \bl{s}^{(ij)}:\bl{s}^{(ij)} \in S_j\}$ of this tensor product space. We will use $k$ to refer to the number of factors in the tensor product and take $p_i \coloneq \text{dim}(V_i)$, where usually $V_i = \mb{R}^{\mc{P}_j}$ with $\mc{P}_j$ a poset containing $p_j$ elements. Shorthand for the set $\{1,\ldots,p\}$ is $[p]$. The standard inner product between two order-$k$ tensors is $\langle \bl{S},\bl{T}\rangle = \sum_{i_1 = 1}^{p_1}\cdots \sum_{i_k = 1}^{p_k}S_{i_1\ldots i_k}T_{i_1 \ldots i_k}$. This induces the Frobenius norm $\Vert \bl{T}\Vert_F^2 = \langle \bl{T}, \bl{T} \rangle$.    

\section{Preliminaries on Low-rank Tensors and Convex Geometry}
\label{sec:background}

For the purposes of this work a tensor $\bl{T} \in \mb{R}^{p_1 \times \cdots \times p_k}$ will be viewed as an array of real numbers $T_{i_1\ldots i_k}$ where $i_j$ ranges from $1$ to $p_j$ for all $j$. Equivalently, a tensor $\bl{T}$ can be viewed as a function on the index set $\times_{j = 1}^k [p_j]$ with $\bl{T}(i_1,\ldots,i_k) \coloneq T_{i_1\ldots i_k}$. We say that such a tensor has order $k$ and dimension $(p_1,\ldots,p_k)$. Order one tensors are vectors and order two tensors are matrices. A large portion of the applications of this work are to matrix-valued data and a reader unacquainted with tensors may restrict their focus to matrices. Modes of the tensor refer to one of the $k$ index sets $[p_j]$, and a fibre of this tensor is a vector that holds every index of $\bl{T}$ fixed except for one. For example, $\bl{S}_{\bullet 4} \in \mb{R}^{p_1}$ represents the fourth column of a matrix --- a mode-one fibre, while $\bl{T}_{3 \, 2  \bullet 4} \in \mb{R}^{p_3}$ is a mode-three fibre of an order four tensor.   

A special class of tensors are tensors that have rank-one, meaning that they can be written multiplicatively as
\begin{align*}
    T_{i_1\ldots i_k} = v^{(1)}_{i_1} \cdots v^{(k)}_{i_k},\;\; \forall (i_1,\ldots,i_k) \in \times_{j = 1}^k [p_j],
\end{align*}
where every $(\bl{v}^{(j)})^\intercal = (v_1^{(j)},\ldots,v_{p_j}^{(j)})$ is a vector in $\mb{R}^{p_j}$. The tensor product notation $\otimes$ is used to denote the above rank-one structure: $\bl{T} = \otimes_{j = 1}^k \bl{v}^{(j)}$. The real (CP \citep[Sec 3]{KoldaTensor}) rank of a tensor is defined as 
\begin{align}
\label{eqn:RankDefinition}
 \text{rank}(\bl{T}) =  \min\bigg\{r: \bl{T} = \sum_{i = 1}^r \otimes_{j = 1}^k \bl{v}^{(ij)}\bigg\},
\end{align}
where every $\bl{v}^{(ij)} \in \mb{R}^{p_j}$. That is, it is the smallest number of rank-one tensors required to express $\bl{T}$ as a sum of rank-one tensors. Whenever rank is referred to in this work it is taken to mean rank over the real numbers. The nonnegative rank \citep{GillisNNMF} places extra nonnegativity constraints on the vectors $\bl{v}^{(ij)}$ and is defined as
\begin{align}
\label{eqn:NNRankDefinition}
 \text{rank}_+(\bl{T}) =  \min\bigg\{r: \bl{T} = \sum_{i = 1}^r \otimes_{j = 1}^k \bl{v}^{(ij)}, \;v^{(ij)}_l \geq 0 \; \forall i,j,l \bigg\}. 
\end{align}
In general, the nonnegative rank of a tensor may exceed its rank \citep{CohenNonNegative,thomas1974rank}.

We will be examining low-rank decompositions where the vectors appearing in the rank-one factors satisfy a conical constraint. A brief summary of some standard ideas from convex geometry are provided below, where the reader is referred to \citep{BoydConvex,ZieglerPolytopes} for more details.

A (convex) cone $\mc{C} \subset \mb{R}^p$ satisfies the property that for all $\bl{x},\bl{y} \in \mc{C}$ and $\alpha,\beta \geq 0$ the vector $\alpha \bl{x} + \beta \bl{y}$ is in $\mc{C}$. The conical hull of the vectors $\bl{v}_1,\ldots,\bl{v}_m$ is defined as $ \{\sum_{i = 1}^m \alpha_i \bl{v}_i: \alpha_i \geq 0\}$. Any cone that can be expressed in this way is called polyhedral with the representation referred to as a vertex or $\mc{V}$-representation. Polyhedral cones can also equivalently be expressed in a halfspace or $\mc{H}$-representation, where a cone is defined by finitely many halfspace constraints: $\{\bl{x}: \bl{a}_i^\intercal \bl{x} \geq 0, i = 1,\ldots,l\}$. The dual cone $\mc{C}^*$ is closely related to the $\mc{H}$-representation as it is defined as the closed, convex set $\mc{C}^* = \{\bl{a}: \bl{a}^\intercal \bl{x} \geq 0, \forall \bl{x} \in \mc{C
}\}$. A cone is said to be pointed if $\mc{C} \cap (-\mc{C}) = \{\bl{0}\}$, solid if it has a non-empty interior, and proper if it is pointed, solid, and closed. Polyhedral cones are pointed if and only if the $\bl{a}_i$ appearing in the $\mc{H}$-representation span $\mb{R}^p$, are solid if and only if the $\bl{v}_i$ appearing in the $\mc{V}$-representation span $\mb{R}^p$, and are always closed. Whenever $\mc{C}$ is polyhedral or proper the dual cone $\mc{C}^*$ inherits these properties respectively. 

One of the simplest types of polyhedral cones are simplicial cones that are generated by $p$ linearly independent vectors $\bl{v}_1,\ldots,\bl{v}_p \in \mb{R}^p$; other classes of proper cones are generated by more than $p$ vectors. The nonnegative orthant $\mb{R}^p_+ = \{\sum_{i = 1}^p \alpha_i\bl{e}_i: \alpha_i \geq 0\}$ is simplicial as it is generated by the standard basis vectors $\bl{e}_i$. A face $\mc{F} \subset \mc{C}$ of a cone is a subset where if $\bl{x} + \bl{y} \in \mc{F}$ and $\bl{x},\bl{y} \in \mc{C}$ then $\bl{x},\bl{y} \in \mc{F}$. One-dimensional faces, referred to as extremal rays, have the form $\{\alpha\bl{v}:\alpha \geq 0\}$ of a ray protruding from $\bl{0}$ in the direction $\bl{v}$. For brevity we will at times refer to $\bl{v}$ as an extremal ray or as being extremal. The extremal rays of $\mb{R}_+^p$ are exactly the coordinate axes $\{\alpha \bl{e}_i:\alpha \geq 0\}$. Any face of $\mb{R}^p_+$ is equal to 
$\{\sum_{i \in I} \alpha_i \bl{e}_i: \alpha_i \geq 0\} = \{\bl{x} \in \mb{R}^p_+: x_i = 0, \forall i \notin I\}$ for some set of indices $I \subset [p]$. The set of faces of a polyhedral cone forms a partially ordered set under set inclusion. For example, the face $\{\alpha_1 \bl{e}_1 + \alpha_2 \bl{e}_2: \alpha_1,\alpha_2 \geq 0\}$ contains the two extremal rays $\{\alpha \bl{e}_1: \alpha \geq 0\}, \{\alpha \bl{e}_2: \alpha \geq 0\}$. As this example illustrates, any face in a proper cone is equal to the conical hull of the extremal rays contained in the face. In particular, if $\mc{C} = \text{cone}(\bl{v}_1,\ldots,\bl{v}_m)$ then the faces of $\mc{C}$ are equal to $\text{cone}(\bl{v}_{i_j}: i_j \in I)$ for some index set $I$, but not every index set $I \subset [m]$ gives rise to a face. A facet of a cone is a face that has dimension equal to the dimension of the cone minus one.

\section{The Nondecreasing Rank}
\label{sec:NonDecRankIntro}
Going beyond nonnegativity constraints, in many applications a tensor that arises from data may also possess some kind of ordering constraint. Consider the following dose-response data in Table \ref{tab:Seleniumdoseresponse} that displays the percentage of flies that died in response to exposure to different types of selenium at different concentrations \citep{Seleniumjeske2009testing,DRCPackage}. As might be expected, there is a general increasing trend within each row of the table, as more flies expire when selenium is present in higher concentrations. It is also apparent from the table that the third type of selenium appears more harmful than the first two types, while the first type may be more harmful than the second, but this is much less clear.

\begin{table}[h!]
    \centering
    \begin{tabular}{|c|cccc|}
\hline 
 & \multicolumn{4}{|c|}{Concentration}
 \\
 \hline
  & 0 & 100 & 200 & 400
    \\
    \hline
    \hline
         Type I & 2.0 & 27.4 & 26.7 & 68.0 \\
        Type II & 1.4 & 19.6 & 41.5 & 40.3 \\
            Type III & 2.9 & 24.4 & 75.0 & 96.5 \\
         \hline
    \end{tabular}
    \caption{Percentages of flies that died after being exposed to three different types of selenium at various concentrations.}
    \label{tab:Seleniumdoseresponse}
\end{table}

If a concise representation of this data was sought, a simple assumption would be that the dose-response functions $f_i:\{0,100,200,400\} \rightarrow [0,100]$ for each type of selenium satisfy $f^{(i)} = c_i f$ for some nondecreasing function $f$ and scalars $c_3 \geq c_1,c_2 \geq 0$, with the latter assumption coming from the observation that the third type of selenium appears to be most harmful. This representation amounts to assuming that the dose-response matrix $\bl{T} \approx \bl{c} \bl{f}^\intercal$ is rank-one and that the nonnegative vectors $\bl{c}$ and $\bl{v}$ are congruous with the respective orderings $\text{Type I, Type II} \leq \text{Type III}$, and $0 \leq 100 \leq 200 \leq 400$ that are imposed on the row and column variables of the matrix. If $\bl{c}$ is scaled so that $\tfrac{1}{3}(c_1 + c_2 + c_3) = 1$ then the vector $\bl{f}$ could be interpreted as the average dose-response curve across all three types of selenium, while $\bl{c}$ reflects the multiplicative, noxious effects of the different varieties of selenium.

Formally, in the selenium example we have associated two partially ordered sets (posets) with the rows and columns of the table, namely $\mc{P}_1 = \{\mathrm{Type\, I, Type \, II, Type \, III}\}$ and $\mc{P}_2 = \{0,100,200,400\}$. The vectors $\bl{c}$ and $\bl{f}$ can be viewed as functions defined on $\mc{P}_1$ and $\mc{P}_2$ respectively. To say that both $\bl{c}$ and $\bl{f}$ respect the given partial orderings is equivalent to saying that these vectors lie in the order cones $\mc{C}(\mc{P}_1)$ and $\mc{C}(\mc{P}_2)$:

\begin{definition}[Order cone and order polytope of functions] Let $\mc{P}$ be a finite poset and take $\bl{f} \in \mb{R}^{\mc{P}}$ be a real-valued function defined on this poset. The function $\bl{f}$ is said to be nondecreasing if $f_x \leq f_y$ whenever $x \preceq y$ in $\mc{P}$. The order cone $\mc{C}(\mc{P})$ consists of all nonnegative, nondecreasing functions on $\mc{P}$, while the order  polytope $\mc{O}(\mc{P})$ consists of all tensors in $\mc{C}(\mc{P})$ with $f_x \leq 1$ for all $x \in \mc{P}$ \citep{StanleyPosetPolytopes}.    
\end{definition}
An order cone is a proper, polyhedral cone that has the $\mc{H}$-representation 
\begin{align*}
    \mc{C}(\mc{P}) = \cap_{x \preceq y}\{\bl{f}: f_x \leq f_y \} \cap \{ \bl{f}: f_x \geq 0, x \in \mc{P}\}.
\end{align*}
More will be said about the cone structure of $\mc{C}(\mc{P})$ in the next section.

It may be the case that a rank-one factorization of the form $\bl{c}\bl{f}^\intercal$ is not sufficient to capture all of the structure inherent to the selenium dose-response matrix; more general low-rank representations are sought. Pairing a monotonicity requirement with a low-rank factorization we now introduce the central notion of interest in this work:
\begin{definition}[Nondecreasing (ND) rank] Let $\bl{T} \in \mb{R}^{\mc{P}_1 \times \cdots \times \mc{P}_k}$ be a tensor defined on the Cartesian product index set $\mc{P} \coloneq \mc{P}_1 \times \cdots \times \mc{P}_k$, where each $\mc{P}_j$ is a finite poset. Any decomposition of the form $\bl{T} = \sum_{i = 1}^r \otimes_{j = 1}^k \bl{v}^{(ij)}$ with $\bl{v}^{(ij)} \in \mc{C}(\mc{P}_j)$ is referred to as a nondecreasing decomposition or factorization. The nondecreasing rank of a tensor $\bl{T}$ is the minimum number of rank-one factors needed in a nondecreasing factorization: 
    \begin{align}
        \mathrm{NDrank}(\bl{T}) = \min\bigg\{r:  \bl{T} = \sum_{i = 1}^r \otimes_{j = 1}^k \bl{v}^{(ij)}, \;\; \bl{v}^{(ij)} \in \mc{C}(\mc{P}_j),  \; \forall i,j  \bigg\}.
    \end{align}
    If $\bl{T}$ does not possess a nondecreasing factorization then it has an infinite ND rank. The set of all tensors with finite ND rank, ND rank equal to $r$, and ND rank less than or equal to $r$ are respectively denoted by $\mc{N}_{< \infty}(\mc{P})$, $\mc{N}_r(\mc{P})$, and $\mc{N}_{\leq r}(\mc{P})$. The poset $\mc{P}$ will generally be dropped from this notation.
\end{definition}
The ND rank refines the nonnegative rank \eqref{eqn:NNRankDefinition} by possibly adding extra constraints, implying that $\mathrm{rank}_+(\bl{T}) \leq \mathrm{NDrank(\bl{T}})$. We note that the terminology of monotone rank has appeared in the literature \citep{TsimermanMonotonerank}, although it is synonymous with the nonnegative rank over an ordered field. Other very general notions of rank that encompass the ND rank as a special case are that of the X-rank \citep{ComonXrank} and the atomic cone rank \citep{AtomicconeRankParrilo}.

The tensor $\bl{T}$ appearing in the definition of ND rank can be viewed as a function on the Cartesian product of posets $\mc{P} \coloneq \times_{j = 1}^k \mc{P}_j$. The set $\mc{P}$ can itself be given the structure of a poset, called the product poset, by defining the ordering with $(x_1,\ldots,x_k) \preceq (y_1,\ldots,y_k)$ if and only if $x_j \preceq y_j$ in $\mc{P}_j$ for every $j$. In the selenium example we have that $(\mathrm{Type \,II},300) \preceq (\mathrm{Type \,III},400)$ but  $(\mathrm{Type \,III},0) \npreceq (\mathrm{Type \,I},400)$. A helpful way to visualize posets is by their $\textbf{Hasse diagrams}$. A Hasse diagram is a directed graph with nodes given by the elements of the poset and arrows drawn from an element $x$ to an element $y$ if and only if $y$ \textbf{covers} $x$, meaning that $x \prec y$ and there does not exist an element $z$ with $x \prec z \prec y$. The covering relation is denoted by $x \lessdot y$.  Figures \ref{fig:HasseDiagramPair} and \ref{fig:HasseDiagramProduct} illustrate the Hasse diagrams of the partial orderings $\mc{P}_1$, $\mc{P}_2$, and $\mc{P}_1 \times \mc{P}_2$ in the selenium example. We see that there is an arrow from $0$ to $100$ as $0 \lessdot 100$, but not from $0$ to $200$ since $100$ lies in between $0$ and $200$. The Hasse diagram of the poset $\mc{P}_2$ is a path and has the additional property of being totally ordered. Totally ordered posets are referred to as \textbf{chains}: for any two elements $x,y\in \mc{P}_2$ either $x \preceq y$ or $y \preceq x$. The poset $\mc{P}_1$ is not a chain since Type I is neither less than or greater than Type II.

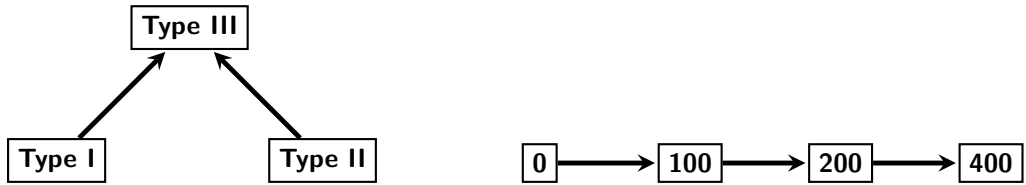
\begin{figure*}[!ht]
    \centering
    \begin{subfigure}[t]{0.5\textwidth}
        \centering
\begin{tikzpicture}[->,>=stealth,shorten >=1pt,auto,node distance=2.5cm,thick,main node/.style={draw,font=\sffamily\small\bfseries}]
  \node[main node] (3) {Type III};
  \node[main node] (1) [below left of=3] {Type I};
  \node[main node] (2) [below right of=3] {Type II};
  \path[every node/.style={font=\sffamily\small},line width=.65mm]
    (1) edge [left] node[left] {} (3)
    (2) edge [right] node[right] {} (3);
\end{tikzpicture}
    \end{subfigure}%
    ~ 
    \begin{subfigure}[t]{0.4\textwidth}
        \centering
\begin{tikzpicture}[->,>=stealth,shorten >=1pt,auto,node distance=2cm,thick,main node/.style={draw,font=\sffamily\small\bfseries}]

  \node[main node] (0) {0};
  \node[main node] (100) [right of=0] {100};
  \node[main node] (200) [right of=100] {200};
    \node[main node] (400) [right of=200] {400};
  \path[every node/.style={font=\sffamily\small},line width=.65mm]
    (0) edge [right] (100)
    (100) edge [right] (200)
        (200) edge [right] (400);
\end{tikzpicture}
    \end{subfigure}
    \caption{The respective Hasse diagrams of $\mc{P}_1 = \{\text{Type I,\, Type II,\, Type III}\}$ and $\mc{P}_2 = \{0,100,200,400\}$.}
    \label{fig:HasseDiagramPair}
\end{figure*}

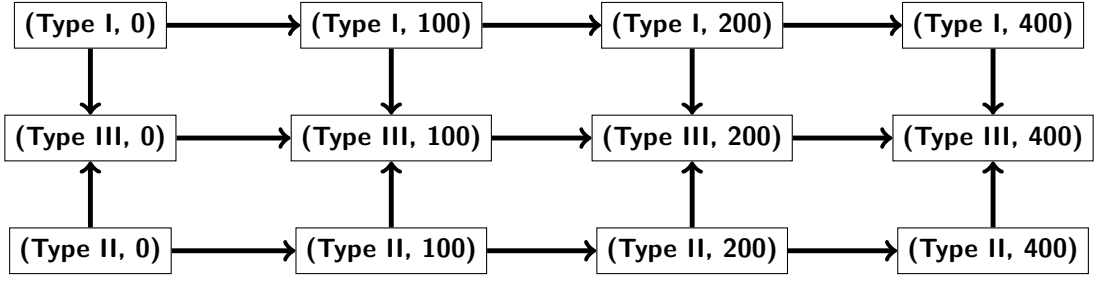
\begin{figure}[!ht]
    \centering
\begin{tikzpicture}[every node/.style={font=\sffamily\small\bfseries}]
 \draw (0,0)  node[draw] (A) {(Type II, 0)};
  \draw (4,0) node[draw] (B) {(Type II, 100)};
    \draw (8,0) node[draw] (C) {(Type II, 200)};
      \draw (12,0) node[draw] (D) {(Type II, 400)};
 \draw (0,1.5)  node[draw] (E) {(Type III, 0)};
  \draw (4,1.5) node[draw] (F) {(Type III, 100)};
    \draw (8,1.5) node[draw] (G) {(Type III, 200)};
      \draw (12,1.5) node[draw] (H) {(Type III, 400)};
 \draw (0,3)  node[draw] (I) {(Type I, 0)};
  \draw (4,3) node[draw] (J) {(Type I, 100)};
    \draw (8,3) node[draw] (K) {(Type I, 200)};
      \draw (12,3) node[draw] (L) {(Type I, 400)};
      \draw [->, thick,line width=.65mm] (A) edge (B) (B) edge (C) (C) edge (D) (E) edge (F) (F) edge (G) (G) edge (H) (I) edge (J) (J) edge (K) (K) edge (L) (I) edge (E) (A) edge (E) (J) edge (F) (B) edge (F) (C) edge (G) (K) edge (G) (D) edge (H) (L) edge (H);
\end{tikzpicture}
    \caption{Hasse Diagram of the product poset $\{\text{Type I,\, Type II,\, Type III}\} \times \{0,100,200,400\}$. }
    \label{fig:HasseDiagramProduct}
\end{figure}

It is worth contemplating why there is a requirement that vectors must be nonnegative in the definition of the order cone, and accordingly in ND factorizations. One motivation for this is that if $\bl{T} = \sum_{i = 1}^r \otimes_{j = 1}^k \bl{v}^{(ij)}$ is an ND factorization then $\bl{T} \in \mc{C}(\mc{P})$ is guaranteed to also respect the order of the product poset $\mc{P}$. If we contrarily permitted a nondecreasing $\bl{v}^{(ij)}$ to contain negative entries the tensor $\otimes_{j = 1}^k \bl{v}^{(ij)}$ would not necessarily be nondecreasing over $\mc{P}$.

\section{Geometry of the Order Cone and the Existence of Nondecreasing Factorizations}
\label{sec:OrderConeGeomExisenceofFactor}

Any nonnegative tensor has a nonnegative factorization. However, the situation differs for monotone tensors; if $\bl{T} \in \mc{C}(\mc{P})$ it is not necessarily the case that $\bl{T}$ has a nondecreasing factorization. In this section and the next we highlight conditions that ensure that a tensor has a finite ND rank. The geometry of order cones and the finite-rank cone will also be examined.

\begin{definition}[Projective Tensor Product of Cones]
    Given convex cones $\mc{C}_j$, $j = 1,\ldots,k$, the projective tensor product \citep{TensorProdConeMulansky} of the $\mc{C}_j$s is equal to
    \begin{align*}
        \otimes_{j = 1}^k \mc{C}_j \coloneq \bigg\{\bl{T}: \exists r, \; \bl{T} = \sum_{i = 1}^r \otimes_{j = 1}^k \bl{v}^{(ij)}, \bl{v}^{(ij)} \in \mc{C}_j\bigg\}.
    \end{align*}
\end{definition}
When $\mc{C}_j = \mc{C}(\mc{P}_j)$ in the above definition the tensor product of cones is the same as the set $\mc{N}_{< \infty}(\times_{i = 1}^k\mc{P}_j)$. The projective tensor product is a proper, polyhedral cone of dimension $\prod_{j = 1}^k p_j$ whenever every $\mc{C}_j \subset \mb{R}^{p_j}$ is proper and polyhedral \citep{TensorProdConeMulansky}. As the projective tensor product is equal to $\text{cone}(\otimes_{j = 1}^k \bl{v}^{(j)}: \bl{v}^{(j)} \in \mc{C}_j)$ the extremal rays must be rank-one tensors. The following result \citep[Thm 3.22]{deBruynTensorProducts} identifies all such extremal rays.

\begin{lemma}
\label{lem:ExtremalRaysTensorProduct}
The extremal rays of $\otimes_{j = 1}^k \mc{C}_j$ consist exactly of the rank-one tensors $\otimes_{j = 1}^k \bl{v}^{(j)}$ where every $\bl{v}^{(j)} \in \mc{C}_j$ is extremal. 
\end{lemma}

In general, the facial structure of $\otimes_{j = 1}^k \mc{C}_j$ is more complicated. An exception is in the special case where almost every $\mc{C}_j$ is simplicial and the dual cone can be described \citep{BarkerTheoryofCones,InjectiveTensorAubrunConeEntagleFeasbility,deBruynTensorProducts}:

\begin{theorem}
\label{thm:InjectiveEqualsProjective}
If at most one of the $\mc{C}_j$ cones is not simplicial then $\big(\otimes_{j = 1}^k \mc{C}_j\big)^* = \otimes_{j = 1}^k \mc{C}_j^*$. If every $\mc{C}_j$ is also polyhedral then the facets of $\otimes_{j = 1}^k \mc{C}_j$ have the form $\{\bl{T}: \langle \bl{T}, \otimes_{j = 1}^k \bl{h}^{(j)}\rangle = 0\} \cap \otimes_{j = 1}^k \mc{C}_j$ where every $\bl{h}^{(j)}$ is extremal in $\mc{C}_j^*$.
\end{theorem}

Using Lemma \ref{lem:ExtremalRaysTensorProduct}, a $\mc{V}$-representation of $\mc{N}_{< \infty}$ can be obtained from the extremal rays of each $\mc{C}(\mc{P}_j)$. The next result \citep[Thm 7]{OrderConeFacesgarcia2020order} describes the extremal rays of any order cone in terms of upsets. An \textbf{upset} $\mc{U} \subseteq \mc{P}$ is a subset with the property that if $x \in \mc{U}$ and $x \preceq y$ then $y \in \mc{U}$. An upset is connected if the subgraph of the Hasse diagram corresponding to the upset is connected.

\begin{lemma}[\cite{OrderConeFacesgarcia2020order}]
\label{lem:OrderConeExtremalRays}
    The extremal rays of $\mc{C}(\mc{P})$ consist vectors of the form $\sum_{i \in I} \bl{e}_i$, where $I$ is a non-empty, connected upset of $\mc{P}$. 
\end{lemma}

Order cones that are simplicial will be of special interest due to Theorem \ref{thm:InjectiveEqualsProjective} and the linear equivalence of simplicial cones and the nonnegative orthant. To describe orders that have simplicial order cones we first introduce some terminology borrowed from the literature on directed graphical models \citep{PearlCausality}.

\begin{definition}
\label{def:Collider}
    A collider in the Hasse diagram of a poset $\mc{P}$ refers to a subgraph of the form
    \begin{center}
        \begin{tikzpicture}[->,>=stealth,shorten >=1pt,auto,node distance=1.5cm,thick,main node/.style={circle, draw, minimum size=.8cm, font=\sffamily\small\bfseries}]
  \node[circle,minimum size= 4pt, main node] (3) {c};
  \node[circle,minimum size= 4pt, main node] (1) [below left of=3] {a};
  \node[circle,minimum size= 4pt, main node] (2) [below right of=3] {b};
  \path[every node/.style={font=\sffamily\small},line width=.65mm]
    (1) edge [left] node[left] {} (3)
    (2) edge [right] node[right] {} (3);
\end{tikzpicture}
    \end{center}

Equivalently, a collider is a set of three elements $a,b,c \in \mc{P}$ where $a \lessdot c$ and $b \lessdot c$.
\end{definition}

\begin{theorem}
\label{thm:SimplicialConeCondition}
            The cone $\mc{C}(\mc{P})$ is simplicial if and only if the Hasse diagram does not contain any colliders. A Hasse diagram has no colliders if and only if as an undirected graph it is a disjoint union of trees where each tree has edges directed away from one its nodes. 
\end{theorem}
\begin{proof}
    If $\mc{P}$ has $p$ elements, by Lemma \ref{lem:OrderConeExtremalRays}, the cone $\mc{C}(\mc{P})$ is simplicial if and only if the number of connected upsets of $\mc{P}$ is $p$.  For every $x \in \mc{P}$, the intervals $[x,\infty) =\{y: x \preceq y\}$ are distinct connected upsets. Thus, the number of connected upsets is at least $p$. To show that the order cone is simplicial it suffices to show that the number of connected upsets within each connected component of size $m$ in the Hasse diagram is equal to $m$.

    Assume that the Hasse diagram contains no colliders. Any upset $\mc{U}$ has the form $\mc{U} = \cup_{u \in \mc{U}} [u,\infty)$. If $\mc{U}$ contains two or more minimal elements $u_1,u_2$ the upset $\mc{U}$ will not be connected since the only undirected path in $\mc{U}$ from $u_1$ to $u_2$ must be directed away from both $u_1$ and $u_2$ since these are minimal elements in $\mc{U}$. Any such path must include a collider, so no such path exists. Every connected upset $\mc{U}$ must therefore have a minimum element $u$, where $\mc{U} = [u,\infty)$. The number of connected upsets in a connected component of size $m$ in this Hasse diagram is equal to $m$. 
    
 Conversely, assume that $\mc{P}$ has a connected component of the Hasse diagram of size $m$ that has a collider. There exists three elements $a,b,c$ with $a,b$ incomparable and $a \prec c$, $b \prec c$. The upset $a \cup b \cup [c,\infty)$ is connected and does not equal $[x,\infty)$ for any $x \in \mc{P}$. This shows that this connected component of $\mc{P}$ has at least $m+1$ connected upsets. 

 The last statement that the absence of colliders implies the given tree structure is proven by noting that if the undirected graph associated with the Hasse diagram contains a cycle, the only way to orient the edges in a cycle to avoid the appearance of a collider is to create a directed cycle. This contradicts the fact that Hasse diagrams are directed acyclic graphs. If a tree does not contain a minimum element that all edges are directed away from there exists at least two distinct minimal elements $a,b$ in the tree. As the tree is connected there exists a path between $a$ and $b$ with arrows directed away from both $a$ and $b$. There is no way to orient the edges of such a path without obtaining a collider. The converse direction that unions of trees with the given structure contain no colliders is apparent. 
\end{proof}


Figure \ref{fig:TreewithNoColliders} illustrates the tree structure described in Theorem \ref{thm:SimplicialConeCondition}, while Figure \ref{fig:OrderConePlot} (a) and (b) respectively illustrate order cones that are and are not simplicial. In the order constrained statistical inference literature a few orderings that commonly occur are chains, the tree order where $x_1 \prec x_2,\ldots,x_p$, and the umbrella order where $x_1 \prec \cdots \prec x_{l-1} \prec x_l \succ x_{l+1} \cdots \succ x_p$ \citep[Sec 2.3]{ConstrainedStatInferenceSilvapuleSen}. The umbrella order represents a kind of unimodality where the mode is known to occur at $x_l$.  Only the umbrella order out of these three orders is not simplicial.

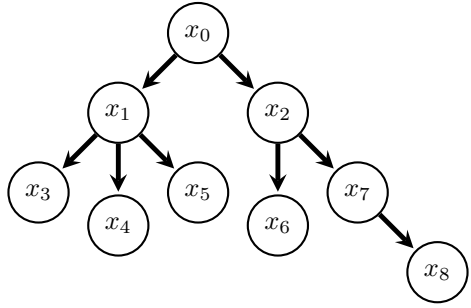
\begin{figure}[H]
    \centering
           \begin{tikzpicture}[->,>=stealth,shorten >=1pt,auto,node distance=1.5cm,thick,main node/.style={circle, draw, minimum size=.8cm, font=\sffamily\small\bfseries}]
  \node[circle,minimum size= 4pt, main node] (0) {$x_0$};
  \node[circle,minimum size= 4pt, main node] (1) [below left of=0] {$x_1$};
  \node[circle,minimum size= 4pt, main node] (2) [below right of=0] {$x_2$};
    \node[circle,minimum size= 4pt, main node] (3) [below left of=1] {$x_3$};
        \node[circle,minimum size= 4pt, main node] (4) [below of=1] {$x_4$};
            \node[circle,minimum size= 4pt, main node] (5) [below right of=1] {$x_5$};
    \node[circle,minimum size= 4pt, main node] (6) [below of=2] {$x_6$};
\node[circle,minimum size= 4pt, main node] (7) [below right of=2] {$x_7$};
\node[circle,minimum size= 4pt, main node] (8) [below right of=7] {$x_8$};
  \path[every node/.style={font=\sffamily\small},line width=.65mm]
    (0) edge [left] node[left] {} (1)
       (1) edge [left] node[left] {} (3)
          (1) edge [left] node[left] {} (4)
             (1) edge [left] node[left] {} (5)
    (2) edge [left] node[left] {} (6)
    (2) edge [left] node[left] {} (7)
        (7) edge [left] node[left] {} (8)
    (0) edge [right] node[right] {} (2);
\end{tikzpicture}
    \caption{A Hasse diagram with a tree shape that has no colliders. }
    \label{fig:TreewithNoColliders}
\end{figure}

\begin{figure}[h!]
    \centering
    \begin{subfigure}[t]{0.5\textwidth}
        \centering
        \includegraphics[width = .8\textwidth]{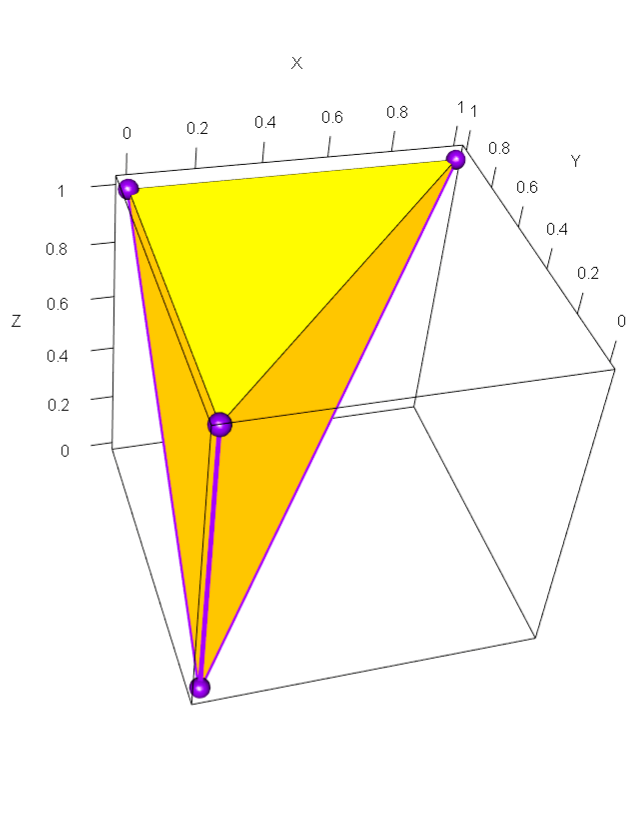}
        \caption{Order polytope of $Z \prec Y \prec X$.}
        \label{fig:OrderConePlotChain}
    \end{subfigure}%
    ~ 
    \begin{subfigure}[t]{0.5\textwidth}
        \centering
        \includegraphics[width = .8\textwidth]{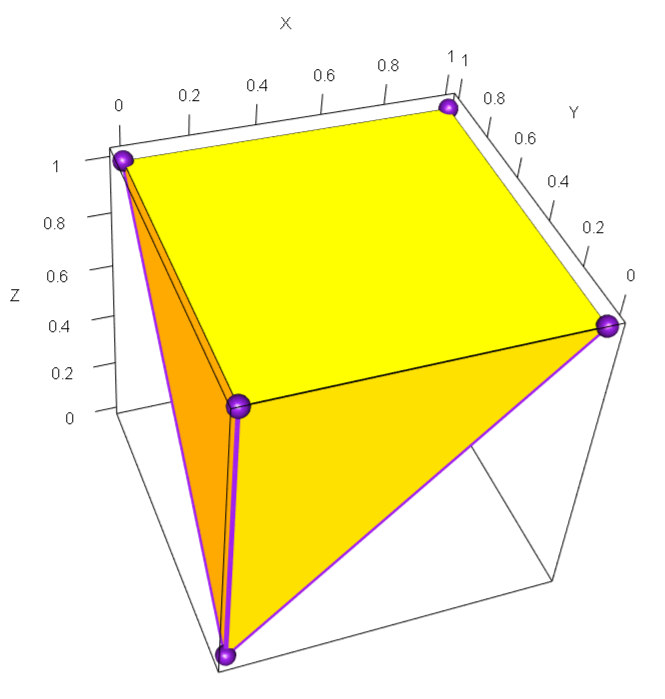}
        \caption{Order polytope of $Z,Y \prec X$.}
        \label{fig:OrderConePlotCollider}
    \end{subfigure}
    \caption{Order polytopes for two different orders. The order cone is found by extending the rays connecting the origin to the top face.}
        \label{fig:OrderConePlot}
\end{figure}

We now pair Lemma \ref{lem:OrderConeExtremalRays} with Lemma \ref{lem:ExtremalRaysTensorProduct} to find a $\mc{V}$-representation of $\mc{N}_{< \infty}$. To illustrate, set $\mc{P}_1 = \{1,2\}$ and $\mc{P}_2 = \{1,2,3\}$, both with the standard ordering. The extremal rays of $\mc{C}(\mc{P}_1 \times \mc{P}_2)$ are all possible ``staircase'' matrices, corresponding to upsets of $\mc{P}_1 \times \mc{P}_2$, that have a staircase of ones in the lower-right corner of the matrix: 
\begin{align}
\label{eqn:ExtremalRaysBox}
    \begin{bmatrix}
        1 & 1 & 1
        \\
        1 & 1 & 1 
    \end{bmatrix},    
        \begin{bmatrix}
        0 & 1 & 1
        \\
        0 & 1 & 1 
    \end{bmatrix},     \begin{bmatrix}
        0 & 0 & 0
        \\
        1 & 1 & 1 
    \end{bmatrix},
    \begin{bmatrix}
        0 & 0 & 0
        \\
        0 & 1 & 1 
    \end{bmatrix},     \begin{bmatrix}
        0 & 0 & 1
        \\
        0 & 0 & 1 
    \end{bmatrix},    \begin{bmatrix}
        0 & 0 & 0
        \\
        0 & 0 & 1 
    \end{bmatrix},
    \\
    \label{eqn:ExtremalRaysPureStaircase}
        \begin{bmatrix}
        0 & 1 & 1
        \\
        1 & 1 & 1 
    \end{bmatrix},
     \begin{bmatrix}
        0 & 0 & 1
        \\
        1 & 1 & 1 
    \end{bmatrix},
    \begin{bmatrix}
        0 & 0 & 1
        \\
        0 & 1 & 1 
    \end{bmatrix}.
\end{align} 
Any connected upset of the poset $\{1,\ldots,m\}$ with the standard ordering has the form $\{a,a+1,\ldots,m\}$ for some $a \in [m]$. Consequently, the extremal rays of $\mc{C}(\mc{P}_1) \otimes \mc{C}(\mc{P}_2)$ have the form $\big(\sum_{i = a}^{2} \bl{e}_i\big)\big(\sum_{j = b}^3 \bl{e}_i)^\intercal$ for $a \in [2]$, $b \in [3]$. These extremal rays are enumerated in the first row of \eqref{eqn:ExtremalRaysBox} and consist of exactly the ``box'' matrices that have a rectangle of ones in the lower-right corner.  The three matrices in \eqref{eqn:ExtremalRaysPureStaircase} are monotone with respect to the partial order $\mc{P}_1 \times \mc{P}_2$ but do not possess ND factorizations. In general, the extremal rays of $\otimes_{j = 1}^k \mc{C}(\mc{P}_j)$ for arbitrary posets $\mc{P}_j$ have the form $\otimes_{j = 1}^k (\sum_{i_j \in I_j} \bl{e}_{i_j})$, where $I_j$ is a non-empty, connected upset of $\mc{P}_j$.

\begin{theorem}
\label{thm:EqualityofMonotoneProjective}
    When $\mc{P} = \times_{j = 1}^k \mc{P}_j$ the cone $\mc{C}(\mc{P})$ is equal to $\mc{N}_{< \infty}$ if and only if every poset $\mc{P}_j$ except for one is trivial. A trivial poset has no arrows in its Hasse diagram. If every $\mc{P}_j = [p_j] = \{1,\ldots,p_j\}$ is a chain, the cone $\mc{N}_{< \infty}$ is generated by $\prod_{j = 1}^k p_j$ extremal rays, while $\mc{C}(\mc{P})$ has the same number of extremal rays as there are non-empty subsets $\mc{A} \subset \mc{P}$, where no two elements of $\mc{A}$ are comparable. In particular, when $k = 2$ the number of extremal rays of $\mc{C}(\mc{P})$ is ${p_1 + p_2 \choose p_1} - 1$. 
\end{theorem}
\begin{proof}
Let $\mc{P}_j$ have size $p_j$ and have $q_j$ connected upsets. The extremal rays of $\mc{N}_{< \infty}$ are all extremal rays of $\mc{C}(\mc{P})$. To see this, observe that every extremal ray of $\mc{N}_{ < \infty}$ has the form $\otimes_{j = 1}^k \bl{v}^{(j)}$ where $\bl{v}^{(j)}$ is extremal in $\mc{C}(\mc{P}_j)$ by Lemma \ref{lem:ExtremalRaysTensorProduct}. By Lemma \ref{lem:OrderConeExtremalRays} each $\bl{v}^{(j)}$ is associated with a connected upset $\mc{U}_j \subset \mc{P}_j$. Any product $\times_{j = 1}^k \mc{U}_j$ of connected upsets is a connected upset in $\mc{P}$. As $\otimes_{j = 1}^k \bl{v}^{(j)}$ is associated with the connected upset $\times_{j = 1}^k \mc{U}_j$, another invocation of Lemma \ref{lem:OrderConeExtremalRays} shows that $\otimes_{j = 1}^k \bl{v}^{(j)}$ is extremal in $\mc{C}(\mc{P})$. 

As a consequence, $\mc{N}_{< \infty} = \mc{C}(\mc{P})$ if and only if these two cones have the same number of extremal rays. Assume that $\mc{P}_1$ is the only non-trivial poset. Then $\times_{j = 1}^k \mc{P}_j$ is a disjoint union of $\prod_{j = 2}^k p_j$ copies of the poset $\mc{P}_1$. The number of connected upsets of $\mc{P}$ is $q_1\prod_{j = 2}^k p_j$ and is equal to the number of extremal rays of $\mc{C}(\mc{P})$. As $p_j = q_j$ for $j > 1$, $q_1\prod_{j = 2}^k p_j = \prod_{i = 1}^k q_j$ is equal to the product of the number of connected upsets of each $\mc{P}_j$, namely the number of extremal rays of $\mc{N}_{< \infty}$. 

To show the other direction, assume that the inequalities $a_j \prec b_j$ are present in $\mc{P}_j$ for $j = 1,2$. 
Letting $x_j$ be a maximal element of $\mc{P}_j$, $j = 3,\ldots,k$, the set
\begin{align*}
    \{(y_1,y_2,x_3,\ldots,x_k): a_1 \preceq y_1, b_2 \preceq y_2 \} \cup     \{(y_1,y_2,x_3,\ldots,x_k): b_1 \preceq y_1, a_2 \preceq y_2\}
\end{align*}
is a connected upset that contains elements of the form $(a_1,b_2,x_3,\ldots,x_k)$ and $(b_1,a_2,x_3,\ldots,x_k)$, but does not contain $(a_1,a_2,x_3,\ldots,x_k)$, implying that it cannot equal a Cartesian product upset $\times_{j = 1}^k \mc{U}_j$. The number of connected upsets of $\mc{P}$ is at least equal to $\prod_{j = 1}^k q_j + 1$ --- the number of connected, product upsets plus the non-product upset described above. Therefore $\mc{C}(\mc{P})$ has more extremal rays than $\prod_{j = 1}^k q_j$, the number of extremal rays of $\mc{N}_{< \infty}$. 

To prove the last result we count the number of connected upsets of $\times_{j = 1}^k[p_j]$. This is equal to the number of non-empty antichains --- sets where no two elements in a set are comparable. There is a bijection between antichains $\mc{A} \subset \mc{P}$ and connected upsets $\cup_{a \in \mc{A}}[a,\infty)$. If $\mc{A} = \{a_1,\ldots,a_m\}$ are the minimal elements of $\mc{U}$ then $\mc{A}$ is an antichain with $\mc{U} = \cup_{a \in \mc{A}} [a,\infty)$ and conversely any antichain $\mc{A}$ produces the upset $\cup_{a \in \mc{A}} [a,\infty)$ with minimal elements $\mc{A}$. This upset, and in fact any upset, is connected when the poset contains a maximum element, as is the case when $\mc{P} = \times_{j = 1}^k [p_j]$ is a product of chains.  When $k = 2$ the number of antichains is equal to the number of staircase matrices; if $(x_1,y_1),\ldots,(x_l,y_l)$ is an antichain then this determines a unique staircase matrix $\bl{M}$, where $M_{ab} = 1$ if and only if there exists an $i = 1,\ldots,l$ with $(a,b) \succeq (x_i,y_i)$.  A staircase matrix is also determined by a path with non-decreasing coordinates in the lattice $\mb{N} \times \mb{N}$ from $(0,0)$ to $(p_1,p_2)$. Such a path must have length $p_1 + p_2$, and out of the $p_1+p_2$ steps along the path, $p_1$ of these steps must be upward in the direction $(1,0)$. The number of such paths is ${p_1 + p_2 \choose p_1}$. As the path $(0,0) - (1,0) - \cdots - (p_1,0) - (p_1,1) - \cdots - (p_1,p_2)$ corresponds to the $\bl{0}$ staircase matrix, or the empty antichain, we must remove this path from consideration. 
\end{proof}

An implication of Theorem \ref{thm:EqualityofMonotoneProjective} is that if a random tensor $\bl{T}$ is drawn from a measure that is absolutely continuous with respect to the Lebesgue measure on $\mc{C}(\mc{P})$ then $\text{Pr}(\bl{T} \in \mc{N}_{< \infty}) > 0$ and $\text{Pr}(\bl{T} \in \mc{C}(\mc{P}) \backslash \mc{N}_{< \infty}) > 0$.  The probability that a $\bl{T}$ drawn from the uniform distribution on the order polytope $\mc{O}(\mc{P})$ lies in $\mc{N}_{< \infty} \cap \mc{O}(\mc{P})$ when $\mc{P} = [m] \times [m]$ is found to equal $50\%, 2.38\%$, and $0.004\%$ when $m = 2,3,4$ by a volume computation. As the dimension of the matrix grows, the number of extremal rays of $\mc{C}(\mc{P})$ becomes significantly larger than that of $\mc{N}_{< \infty}$, and the probability of observing a random monotone tensor with finite ND rank is small. This suggests that when a data matrix $\bl{T} = \bl{X} + \bs{\epsilon}$ is observed by taking a finite ND rank matrix $\bl{X}$ and perturbing it by noise $\bs{\epsilon}$, there is only a small probability that the observed $\bl{T}$ has finite ND rank. From a parameter estimation perspective this does not pose a serious issue as the observed $\bl{T}$ can be projected onto $\mc{N}_{< \infty}$ or $\mc{N}_{\leq r}$. We consider such low-rank approximations in Section \ref{sec:Optimization}.  

Given a tensor $\bl{T}$, it must first be ascertained that $\bl{T} \in \mc{N}_{< \infty}$ if $\bl{T}$ is to have an exact ND factorization. Equations for determining if $\bl{T} \in \mc{N}_{< \infty}$ can be found algorithmically by converting the $\mc{V}$-representation of  $\mc{N}_{< \infty}$ in Lemma \ref{lem:ExtremalRaysTensorProduct} into a $\mc{H}$-representation. In the next section we provide simple equations for testing membership when every $\mc{P}_i$ except for one has a tree structure. Below, the $24$ equations of $\mc{N}_{< \infty} \subset \mb{R}_+^{3 \times 3}$ are computed in Polymake \citep{gawrilow2000Polymake} for the case where $\mc{P}_i = \{a^{(i)},b^{(i)},c^{(i)}\}$, $i = 1,2$ is given by the collider structure specified in Definition \ref{def:Collider}, and the third row and column respectively correspond to the maximum elements $c^{(1)},c^{(2)}$.

\begin{align*}
-t_{11} \leq 0,\; 
-t_{12} \leq 0,\;
-t_{21} \leq 0,\;
-t_{22} \leq 0,\\
t_{21} - t_{23} \leq 0,\;
t_{12} - t_{32} \leq 0,\;
t_{22} - t_{23} \leq 0,\;
t_{11} - t_{31} \leq 0,\\
t_{12} - t_{13} \leq 0,\;
t_{22} - t_{32} \leq 0,\;
t_{21} - t_{31} \leq 0,\;
t_{11} - t_{13} \leq 0,\\
    -t_{11} + t_{13} + t_{31} - t_{33} \leq 0, \;
-t_{22} + t_{23} + t_{32} - t_{33} \leq 0,\\
-t_{21} + t_{23} + t_{31} - t_{33} \leq 0, \;
-x_2 + t_{13} + t_{32} - t_{33} \leq 0,\\
t_{11} + t_{12} - t_{13} + t_{21} - t_{22} - t_{31} \leq 0,\;
t_{11} + t_{12} - t_{13} - t_{21} + t_{22} - t_{32} \leq 0,\\
t_{11} - t_{12} + t_{21} + t_{22} - t_{23} - t_{31} \leq 0,\;
-t_{11} + t_{12} + t_{21} + t_{22} - t_{23} - t_{32} \leq 0,\\
-t_{11} - t_{12} + t_{13} - t_{21} + t_{22} + t_{31} - t_{33} \leq 0, 
-t_{11} - t_{12} + t_{13} + t_{21} - t_{22} + t_{32} - t_{33} \leq 0,\\
-t_{11} + t_{12} - t_{21} - t_{22} + t_{23} + t_{31} - t_{33} \leq 0,\;
t_{11} - t_{12} - t_{21} - t_{22} + t_{23} + t_{32} - t_{33} \leq 0.
\end{align*}
The first three rows of equations correspond to simple positivity and monotonicity constraints. The next two rows of equations have four non-zero variables and will appear in the next section as the type of equations required to cut-out $\mc{N}_{< \infty}$ when the posets are trees. The final four rows of equations are more exotic as they can be written as $\tr(\bl{C}^\intercal\bl{T}) \leq 0$ where the matrix $\bl{C}$ has a rank greater than one.  For example, the last two inequalities have the corresponding $\bl{C}$ matrices
\begin{align*}
    \begin{bmatrix}
        1 & -1 & 0
        \\
        1 & 1 & -1
        \\
        -1 & 0 & 1
    \end{bmatrix}, \;     \begin{bmatrix}
        -1 & 1 & 0
        \\
        1 & 1 & -1
        \\
        0 & -1 & 1
    \end{bmatrix}.
\end{align*}
These matrices do not lie in $\otimes_{j = 1}^2 \mc{C}(\mc{P}_j)^*$, for if they did they would be extremal in $\otimes_{j = 1}^2 \mc{C}(\mc{P}_j)^*$ as they are normal vectors of facets of $\otimes_{j = 1}^2\mc{C}(\mc{P}_j)$, and by Lemma \ref{lem:ExtremalRaysTensorProduct} they would be rank one matrices.

\section{Reduction of the ND rank to Nonnegative Rank}
\label{sec:EquivofNDandNNRanks}
For invertible linear maps $\bl{A}^{(j)}$, we define the tensor product of these linear maps as the linear map $\bl{A}^{(1)} \otimes \cdots \otimes \bl{A}^{(k)}:\mb{R}^{p_1 \times \cdots \times p_k} \rightarrow \mb{R}^{p_1 \times \cdots \times p_k}$ that is given in coordinates by
\begin{align}
\label{eqn:TensorProductofMatrices}
    [(\bl{A}^{(1)} \otimes \cdots \otimes \bl{A}^{(k)})(\bl{T})]_{j_1\ldots j_k} = \sum_{i_1 = 1}^{p_1}\cdots \sum_{i_k = 1}^{p_k} A_{j_1 i_1}^{(1)} \cdots A_{j_ki_k}^{(k)} T_{i_1\ldots i_k}.
\end{align}
On rank-one tensors this linear map is given by the simple formula $(\otimes_{j = 1}^k \bl{A}^{(j)})(\otimes_{j = 1}^k \bl{v}^{(j)}) = \otimes_{j = 1}^k \bl{A}^{(j)}\bl{v}^{(j)}$. The idea in this section is to choose the maps $\bl{A}^{(j)}$ so that the order cone $\mc{C}(\mc{P}_j)$ is mapped bijectively onto the positive orthant $\mb{R}^{p_j}_+$ by $\bl{A}^{(j)}$. As the tensor product map preserves rank-one tensors, the problem of determining the ND rank can be transformed into a problem of determining the nonnegative rank whenever such $\bl{A}^{(j)}$ exist. 


\begin{theorem}
\label{thm:NDNonnegativeFactEquivalence}
    Assume that $\mc{C}(\mc{P}_j)$ is simplicial for every $j = 1,\ldots,k$ and let $\bl{A}^{(j)} \in \mathrm{GL}(p_j)$ be an invertible linear map where $\bl{A}^{(j)}(\mc{C}(\mc{P}_j)) = \mb{R}^{p_j}_+$. If $\bl{T} \in \mc{N}_{< \infty}$ then $\mathrm{NDrank}(\bl{T}) = \mathrm{rank}_+\big(\otimes_{j = 1}^k \bl{A}^{(j)}(\bl{T})\big)$.    
\end{theorem}
\begin{proof}
    As $\mc{C}(\mc{P}_j)$ is simplicial it is the conical hull of $\bl{v}^{(1)},\ldots,\bl{v}^{(p_j)}$. We can define $\bl{A}^{(j)}$ to be the linear map that sends $\bl{v}^{(i)}$ to $\bl{e}_i$ and thus maps $\mc{C}(\mc{P}_j)$ onto $\mb{R}^{p_j}_+$. If $\bl{T} = \sum_{i = 1}^r \otimes_{j = 1}^k \bl{v}^{(ij)}$ is a rank-$r$ ND factorization of $\bl{T}$ then $\otimes_{j = 1}^k \bl{A}^{(j)}(\bl{T}) = \sum_{i = 1}^r \otimes_{j = 1}^k \bl{A}^{(j)}\bl{v}^{(ij)}$ is a nonnegative rank-$r$ tensor factorization of $\otimes_{j = 1}^k \bl{A}^{(j)}(\bl{T})$. Conversely, if $\otimes_{j = 1}^k \bl{A}^{(j)}(\bl{T}) = \sum_{i = 1}^r \otimes_{j = 1}^k \bl{w}^{(ij)}$ is a rank-$r$ nonnegative factorization of $\otimes_{j = 1}^k \bl{A}^{(j)}(\bl{T})$ then $\sum_{i = 1}^r \otimes_{j = 1}^k (\bl{A}^{(j)})^{-1}\bl{w}^{(ij)}$ is a rank-$r$ ND factorization of $\bl{T}$.  
\end{proof}

By Lemma \ref{thm:SimplicialConeCondition}, $\mc{C}(\mc{P}_j)$ is simplicial when the Hasse diagram of $\mc{P}_j$ is a union of trees. One special case where the above theorem applies is when $\mc{P} = \times_{j = 1}^k [p_j]$ under the standard ordering; the problem of finding an exact ND factorization can be converted into the problem of finding a nonnegative factorization. The matrices $\bl{A}^{(j)}$ required for this conversion are related to M\"obius inversions. The M\"obius tranform $\bl{M}_{\mc{P}}:\mb{R}^{\mc{P}} \rightarrow \mb{R}^{\mc{P}}$ of a poset $\mc{P}$ is defined as the linear map with
\begin{align*}
    \bl{M}_{\mc{P}}(\bl{e}_x) = \sum_{x \preceq y} \bl{e}_y, \;\; \forall \bl{x} \in \mc{P},
\end{align*}
where $\{\bl{e}_x\}_{x \in \mc{P}}$ is the standard basis of $\mb{R}^{\mc{P}}$.

\begin{lemma}
\label{lem:MobiusDirectSumandTensor}
 If $\mc{P}$ has connected components $\mc{S}_1,\ldots,\mc{S}_m$ then $\bl{M}_{\mc{P}} = \oplus_{i = 1}^m \bl{M}_{\mc{S}_i}$. Moreover, if $\mc{P} = \times_{j = 1}^k \mc{P}_j$ then $\bl{M}_{\mc{P}} = \otimes_{j = 1}^k \bl{M}_{\mc{P}_j}$. 
\end{lemma}
\begin{proof}
We first prove that $\bl{M}_{\mc{P}} = \oplus_{i = 1}^m \bl{M}_{\mc{S}_i}$, meaning that $\bl{M}_{\mc{P}}$ is a block-diagonal matrix with $m$ diagonal blocks that equal the $\bl{M}_{\mc{S}_i}$s.  If $x \in \mc{S}_i$ then $\bl{M}_{\mc{P}}\bl{e}_x = \sum_{y \in \mc{P}: x \preceq y} \bl{e}_y =  \sum_{y \in \mc{S}_i: x \preceq y} \bl{e}_y = \bl{M}_{\mc{S}_i}\bl{e}_x$, where we have used the fact that the only elements in $\mc{P}$ that are comparable to $x$ are in $\mc{S}_i$. This proves the first statement. The second statement follows from
    \begin{align*}
        \bl{M}_{\mc{P}}\bl{e}_{(x_1,\ldots,x_k)} = \sum_{ (x_1,\ldots,x_k) \preceq (y_1,\ldots,y_k)} \bl{e}_{(x_1,\ldots,x_k)} = \sum_{x_1 \preceq y_1}\cdots \sum_{x_k \preceq y_k} \bl{e}_{(y_1,\ldots,y_k)} = (\otimes_{j = 1}^k \bl{M}_{\mc{P}_j}) (\bl{e}_{(x_1,\ldots,x_k)}), 
    \end{align*}
    where $\bl{e}_{(x_1,\ldots,x_k)} = \otimes_{j = 1}^k \bl{e}_{x_j}$ is a rank-one, standard basis element of $\mb{R}^{\mc{P}} = \otimes_{j = 1}^k \mb{R}^{\mc{P}_j}$. That $\bl{M}_{\mc{P}} = \otimes_{j = 1}^k \bl{M}_{\mc{P}_j}$ is a consequence of the tensor product definition for linear maps provided in \eqref{eqn:TensorProductofMatrices}. 
\end{proof}

\begin{lemma}
\label{lem:SimplicialMobiusMapsOrthant}
        If $\mc{P}$ is simplicial, $\bl{M}_{\mc{P}}^{-1}$ bijectively maps $\mc{C}(\mc{P})$ onto $\mb{R}^p_+$. 
\end{lemma}
\begin{proof}
When $\mc{C}(\mc{P})$ is simplicial every upset $\mc{U}$ of $\mc{P}$ has the form $\mc{U} = [x,\infty)$ by the proof in Theorem \ref{thm:SimplicialConeCondition}. The column corresponding to $\bl{e}_x$ in $\bl{M}_{\mc{P}}$ is equal to $\sum_{y \in [x, \infty)} \bl{e}_y$, which is one of the extremal rays of $\mc{C}(\mc{P})$ by Lemma \ref{lem:OrderConeExtremalRays}. Thus, the columns of $\bl{M}_{\mc{P}}$ are exactly the extremal rays of $\mc{C}(\mc{P})$ and so $\bl{M}_{\mc{P}}$ maps $\mb{R}^p_+$ bijectively onto $\mc{C}(\mc{P}) = \text{cone}([\bl{M}_{\mc{P}}]_{\cdot 1},\ldots,[\bl{M}_{\mc{P}}]_{\cdot p})$.
\end{proof}

 \begin{lemma}
    If $\mc{C}(\mc{P})$ is simplicial, the matrix $\bl{M}_{\mc{P}}^{-1}$ has entries
    \begin{align*}
     [\bl{M}_{\mc{P}}^{-1}]_{xy} =   \begin{cases}
            1, \;\; &x = y
            \\
            -1, \;\; &y \lessdot x
            \\
          0, \;\; &\mathrm{otherwise} 
        \end{cases}.
    \end{align*}
    When $\mc{P} = [p] = \{1,\ldots,p\}$ is a chain the above formula simplifies to a Toeplitz matrix of the form:
     \begin{align}
          \label{eqn:MobiusInversionofChain}
         \bl{M}_{\mc{P}}^{-1} = \begin{bmatrix}
             1 & 0 & 0 & 0 &\cdots & 0 
             \\
             -1 & 1 & 0 & 0 & \cdots & 0
             \\\
             0 & -1 & 1 & 0 & \cdots & 0 
             \\
             \vdots & \ddots & \ddots & \ddots & \ddots & \vdots 
                          \\
             \vdots & \ddots & \ddots & \ddots & \ddots & 0 
             \\
             0 & \cdots & \cdots & 0 & -1 & 1
         \end{bmatrix},
     \end{align}
     where the columns and rows are ordered with respect to the basis $\bl{e}_1,\ldots,\bl{e}_p$.
 \end{lemma}
\begin{proof}
We compute 
    \begin{align*}
        [\bl{M}_{\mc{P}} \bl{M}_{\mc{P}^{-1}}]_{\cdot y} & = \sum_{z \in \mc{P}}   [\bl{M}_{\mc{P}}]_{\cdot z}  [\bl{M}_{\mc{P}}^{-1}]_{zy} = [\bl{M}_{\mc{P}}]_{\cdot y} - \sum_{z: y \lessdot z} [\bl{M}_{\mc{P}}]_{\cdot z}
        \\
        & = \sum_{x \in [y,\infty)} \bl{e}_x - \sum_{z: y \lessdot z} \sum_{x \in [z,\infty)} \bl{e}_x = \bl{e}_y.
    \end{align*}
    The last equality follows from the assumed tree structure of $\mc{P}$, where the upset $[y,\infty)$ consists of $y$ along with the upsets of all elements $z$ that cover $y$. If $z_1,z_2$ both cover $y$ then $[z_1,\infty) \cap [z_2,\infty) = \emptyset$ since $\mc{P}$ has no colliders. This ensures that there are no repeats of basis elements in the double sum above.  
\end{proof}

When each $\mc{P}_j = [p_j] = \{1,\ldots,p_j\}$ is a chain, Lemma \ref{lem:SimplicialMobiusMapsOrthant} implies that $\bl{M}_{\mc{P}_j}^{-1}$ bijectively maps $\mc{C}(\mc{P}_j)$ onto $\mb{R}^{p_j}_+$, while Lemma \ref{lem:MobiusDirectSumandTensor} further implies that $\otimes_{j = 1}^k \bl{M}_{\mc{P}_j}^{-1}$ maps $\mc{N}_{< \infty}$ bijectively onto $\otimes_{j = 1}^k \mb{R}^{p_j}_+$. An explicit formula for $\bl{M}_{\mc{P}_j}^{-1}$ is provided in \eqref{eqn:MobiusInversionofChain}. This provides a certificate of whether $\bl{T}$ is in $\mc{N}_{< \infty}$, since this occurs if and only if the entries of $(\otimes_{j = 1}^k \bl{M}_{\mc{P}_j}^{-1})(\bl{T})$ are nonnegative. The $(i_1,\ldots,i_k)$ entry of this tensor is
 \begin{align}
 \label{eqn:Tensordifference}
     [(\otimes_{j = 1}^k \bl{M}_{\mc{P}_j}^{-1})(\bl{T})]_{i_1\ldots i_k} = \sum_{j_1 = 0}^1 \cdots \sum_{j_k = 0}^1 (-1)^{\sum_{l = 1}^k j_l} T_{i_1 - j_1,\ldots,i_k - j_k},
 \end{align}


 where we use the convention that $T_{i_1 - j_1,\ldots,i_k - j_k} = 0$ whenever $i_l - j_l = 0$ for any $l$. Viewing $\bl{T}$ as a function, we can define the differencing operator along the $j$th mode as $(\Delta^{(j)}\bl{T})(i_1,\ldots,i_k) = T(i_1,\ldots,i_j,\ldots,i_k) - T(i_1,\ldots,i_j - 1,\ldots,i_k)$. From \eqref{eqn:Tensordifference} the tensor $(\otimes_{j = 1}^k \bl{M}_{\mc{P}_j}^{-1})(\bl{T})$ can be seen to be equal to the tensor $\Delta^{(1)}\cdots \Delta^{(k)} \bl{T}$. In the matrix setting, with $p_1 = 2$, $p_2 = 3$ we obtain
 \begin{align*}
   (\otimes_{j = 1}^2 \bl{M}_{\mc{P}_j}^{-1})(\bl{T}) &=   \Delta^{(1)}\Delta^{(2)}\bigg(\begin{bmatrix}
         t_{11} & t_{12} & t_{13}
         \\
         t_{21} & t_{22} & t_{23}
     \end{bmatrix}\bigg)  = \Delta^{(1)}\bigg(\begin{bmatrix}
         t_{11} & t_{12} - t_{11} & t_{13} - t_{12}
         \\
         t_{21} & t_{22} - t_{21} & t_{23} - t_{22}
     \end{bmatrix}\bigg) 
     \\
     & = \begin{bmatrix}
           t_{11} & t_{12} - t_{11} & t_{13} - t_{12}
           \\
           t_{21} - t_{11} & t_{22} - t_{21} - t_{12} + t_{11} & t_{23} - t_{22} - t_{13} + t_{12}
     \end{bmatrix}.
 \end{align*}
Equations of the above form were highlighted in \citep{ResselMonotonicFunctions}, where monotonicity for functions defined on a Cartesian product of intervals were examined. The notion of a $\bl{1}_k$ increasing tensor in \citep{ResselMonotonicFunctions} is equivalent to the present notion of a finite ND rank tensor. Notice that similar expressions appeared in the previous section for the $\mc{H}$-representation for a product of collider posets.

We now discuss a probabilistic interpretation of the M\"obius transform of the poset $\mc{P} = \times_{j = 1}^k [p_j]$ and Theorem \ref{thm:NDNonnegativeFactEquivalence}. Suppose that $\bl{R}$ is a nonnegative tensor with entries summing to one. The entry $R_{i_1 \ldots i_k}$ can be interpreted as a probability, where $\bl{R}$ is a probability mass function (PMF) on $\mc{P} = \times_{j = 1}^k [p_j]$. Any rank-$r$ nonnegative tensor factorization of $\bl{R}$ can be written in the form $\sum_{i = 1}^r \lambda_i \otimes_{j = 1}^k \bl{q}^{(ij)}$ with vectors $\bs{\lambda} \in \Delta_{r-1}$ and $\bl{q}^{(ij)} \in \Delta_{p_j-1}$ all residing in probability simplices. This means that $\bl{R}$ is a mixture of $r$ probability mass functions $\otimes_{j = 1}^k \bl{q}^{(ij)}$ on $\mc{P}$, with respective mixture weights $\lambda_i$ \citep[Ch 4]{DrtonLecturesonAlgstats}. The tensor $\bl{M}_{\mc{P}}(\bl{R})$ is exactly the multivariate conditional distribution function (CDF) of $\bl{R}$, meaning that if $\bl{X} \in \mc{P}$ is a random variable with distribution given by $\bl{R}$ then $\bl{M}_{\mc{P}}(\bl{R})_{i_1\ldots i_k} = \text{Pr}( \bl{X} \preceq (i_1,\ldots,i_k))$. If $\bl{F}_i \in \mc{N}_{1}$ is the CDF corresponding to $\otimes_{j = 1}^k \bl{q}^{(ij)}$ then $\bl{M}_{\mc{P}}(\bl{R}) = \sum_{i = 1}^r \lambda_i \bl{F}_i$ has an ND rank of $r$. In summary, a probability distribution has a PMF with a nonnegative rank of $r$ if and only if it has a CDF with a ND rank of $r$.

To conclude this section we provide the $\mc{H}$-representation for $\mc{N}_{< \infty}$ when all but one cone $\mc{C}(\mc{P}_j)$ is simplicial. The following theorem generalizes equation \eqref{eqn:Tensordifference} beyond the case where each $\mc{P}_j$ is a chain. Furthermore, the left-hand side of equation \eqref{eqn:CoveringDifferenceHalfspace} is the formula for obtaining a PMF from the CDF $\bl{T}$ when each $\mc{P}_j$ is a chain.

\begin{theorem}
Assume that $\mc{P}_j$ has no colliders for all but one $j = 1,\ldots,k$. Define the poset $\mc{P}_j' = \{0_j\} \cup \mc{P}_j$ where $0_j \prec x$ for all $x \in \mc{P}_j$. For every pair of elements $(x_j,y_j)$ in $\mc{P}_j$ with $y_j$ covering $x_j$ define the inequality
\begin{align}
\label{eqn:CoveringDifferenceHalfspace}
    \sum_{i_1 \in  \{x_1,y_1\}} \cdots \sum_{i_k \in \{x_k,y_k\}} (-1)^{\sum_{j = 1}^k \delta(i_j)}  T_{i_1 \ldots i_k} \geq 0,
\end{align}
where $\delta(i_j) = 0$ if $i_j = y_j$ and $1$ when $i_j = x_j$. The set of finite ND rank tensors with respect to $\times_{j = 1}^k\mc{P}_j$ is the cone defined by the intersection of all of the halfspaces \eqref{eqn:CoveringDifferenceHalfspace}. Each hyperplane in \eqref{eqn:CoveringDifferenceHalfspace} supports a facet of $\otimes_{j = 1}^k \mc{C}(\mc{P}_j)$.    
\end{theorem}
\begin{proof}
By Theorem \ref{thm:InjectiveEqualsProjective} $\big(\otimes_{j = 1}^k \mc{C}(\mc{P}_j)\big)^* = \otimes_{j = 1}^k \mc{C}(\mc{P}_j)^*$. The facets of $\otimes_{j = 1}^k \mc{C}(\mc{P}_j)$ are supported by hyperplanes corresponding to the extremal rays of $\big(\otimes_{j = 1}^k \mc{C}(\mc{P}_j)\big)^*$. Lemma \ref{lem:ExtremalRaysTensorProduct} shows that the extremal rays of $\big(\otimes_{j = 1}^k \mc{C}(\mc{P}_j)\big)^* $ equal $\otimes_{j = 1}^k \bl{h}^{(j)}$ where each $\bl{h}^{(j)}$ is an extremal ray of $\mc{C}(\mc{P}_j)^*$. The set $\mc{C}(\mc{P}_j)$ is the intersection of the halfspaces of the form $\tilde{H}_{xy} = \{\bl{f}: f_x \leq f_y\}$ for $x \preceq y$ and $x,y \in \mc{P}_j'$. In the case where $x = 0$ we define $\tilde{H}_{0y} = \{\bl{f}:0 \leq f_y\}$ and we let $H_{xy}$ be the hyperplane associated with the halfspace $\tilde{H}_{xy}$. The intersection $H_{xy} \cap \mc{C}(\mc{P}_j)$ is only $(p_j-1)$-dimensional when $x \lessdot y$ in $\mc{P}_j'$, as if $x \prec z \prec y$ then $f_x = f_z = f_y$ for all $\bl{f}$ in $H_{xy}$ (where we have also defined $f_0 \coloneq 0$) and $H_{xy} \cap \mc{C}(\mc{P}_j)$ has dimension less than $p_j - 1$. Conversely, if $x \lessdot y$ then $H_{ab} \cap \mc{C}(\mc{P}_j)$ has dimension $p_j - 1$. 

To complete the proof, let $\bl{h}^{(j)} = \bl{e}_{y_j}^* - \bl{e}_{x_j}^*$ for $x_j \lessdot y_j$ in $\mc{P}_j$, where $\bl{e}_0 \coloneq \bl{0}$ and each $\bl{e}_{x}^*$ is the dual vector for $\bl{e}_{x}$ with respect to the standard bases $\{\bl{e}_x\}_{x \in \mc{P}_j}$ of each $\mc{P}_j$. The hyperplane defining the facet $\otimes_{j = 1}^k \bl{h}^{(j)}$ is 
\begin{align*}
    \bigg\{\bl{T}: \otimes_{j = 1}^k (\bl{e}_{y_j}^* - \bl{e}_{x_j}^*)(\bl{T}) = \sum_{i_1 \in \{x_1,y_1\}} \cdots \sum_{i_k \in \{x_k,y_k\}} (-1)^{\sum_{j = 1}^k \delta(i_j) }T_{i_1 \ldots i_k} \geq 0 \bigg\}.
\end{align*}
\end{proof}

From this theorem we see that the matrix appearing in Table \ref{tab:Seleniumdoseresponse} in the selenium example does not possess an exact ND decomposition. This is because the constraints, $t_{12} \leq t_{32}$, $ t_{11} - t_{12} - t_{31} + t_{32} \geq 0$, and $t_{13} - t_{14} - t_{33} + t_{34} \geq 0$, all of which are necessary for a ND factorization to exist, do not hold.

\section{Maximum and Typical ND Ranks}
\label{sec:MaxTypicalRanks}
\subsection{The Maximum ND Rank}

The  maximum (finite) ND rank associated with the order cones $\mc{C}(\mc{P}_j)$, $j = 1,\ldots,k$ is defined as
\begin{align*}
    \text{maxNDrank} = \sup_{r < \infty} \{\mc{N}_r \neq \emptyset \}.
\end{align*}
In this section the maximum ND rank is found for certain order cones in the matrix setting. Even in the matrix case, finding the maximum ND rank for general order cones is a challenging problem. In comparison, the maximum nonnegative rank, is known to equal $\prod_{j = 1}^{k-1} p_j$ for tensors in $\otimes_{j = 1}^k \mb{R}^{p_j}_+$ where $p_1 \leq \cdots \leq p_k$ \citep{sumi2018maximal}. An analogous upper bound on the maximum ND rank is found below.

\begin{lemma}
\label{lem:NDrankUpperBound}
 Let $\mc{C}(\mc{P}_j)$ have $q_j$ extremal rays where $q_1 \leq \cdots \leq q_k$. The maximum ND rank in $\otimes_{j = 1}^k \mc{C}(\mc{P}_j)$ is at most $\prod_{j = 1}^{k-1} q_j$. 
\end{lemma}
\begin{proof}
    Let $\bl{T} \in \otimes_{j = 1}^k \mc{C}(\mc{P}_j)$ have the ND factorization $\sum_{i = 1}^r \otimes_{j = 1}^k \bl{v}^{(ij)}$. Letting $\bl{w}^{(1j)},\ldots, \bl{w}^{(q_jj)}$ be the extremal rays of $\mc{C}(\mc{P}_j)$, there exists $\alpha_{ijl}$ with $\bl{v}^{(ij)} = \sum_{l = 1}^{q_j} \alpha_{ijl} \bl{w}^{(lj)}$. We have that
    \begin{align*}
        \sum_{i = 1}^r \otimes_{j = 1}^k \bl{v}^{(ij)} & = \sum_{i = 1}^r \bigotimes_{j = 1}^k \bigg(\sum_{l = 1}^{q_j} \alpha_{ijl} \bl{w}^{(lj)} \bigg)
        \\
        & = \sum_{l_1 = 1}^{q_1} \cdots \sum_{l_{k-1} = 1}^{q_{k-1}}  \bigg( \bigotimes_{j = 1}^{k-1} \bl{w}^{(l_jj)} \otimes \bigg(\sum_{i = 1}^r \sum_{l_k = 1}^{q_k} \big(\prod_{j = 1}^k\alpha_{ijl_j}\big) \bl{w}^{(l_k k)}\bigg)\bigg), 
    \end{align*}
    which is a sum of $\prod_{j = 1}^{k-1} q_j$ rank one terms in $\otimes_{j = 1}^k \mc{C}(\mc{P}_j)$. 
\end{proof}

If there are a large number of extremal rays in the given order cones the following lemma can provide a tighter upper bound on the maximum ND rank.

\begin{lemma}
    The maximum ND rank is at most $\prod_{j = 1}^k p_j$, where $\mc{C}(\mc{P}_j) \subset \mb{R}^{p_j}$.
\end{lemma}
\begin{proof}
    By Caratheodory's theorem, for any tensor $\bl{T} \in \otimes_{j = 1}^k \mc{C}(\mc{P}_j) = \text{cone}(\otimes_{j = 1}^k \bl{v}^{(j)}:\bl{v}^{(j)} \in \mc{C}(\mc{P}_j))$ the number of elements in $\{\otimes_{j = 1}^k \bl{v}^{(j)}:\bl{v}^{(j)} \in \mc{C}(\mc{P}_j)\}$ needed to write $\bl{T}$ as a conic combination is at most $\prod_{j = 1}^k p_j$. 
\end{proof}

To find a lower bound on the maximum ND rank we use an analogue to the nested cone condition \citep[Sec 2.1.1]{GillisNNMF} for nonnegative matrix factorizations.

\begin{lemma}
\label{lem:MaxRankDualSpaceImageLowBound}
    The ND rank of a matrix $\bl{T} \in \mc{N}_{< \infty}$ is at least equal to the minimum number of elements $\bl{v}_1,\ldots,\bl{v}_s \in \mc{C}(\mc{P}_1)$ needed for $\bl{T}(\mc{C}(\mc{P}_2)^*) \subseteq \mathrm{cone}(\bl{v}_1,\ldots,\bl{v}_s)$. 
\end{lemma}
\begin{proof}
    Suppose that $\bl{T} = \sum_{i = 1}^r \bl{a}_i \otimes \bl{b}_i$ is a ND factorization. Then $\bl{T}(\mc{C}(\mc{P}_2)^*) \subseteq \mathrm{cone}(\bl{a}_1,\ldots,\bl{a}_r)$ since $\bl{T}(\bs{\beta}) = \sum_{i = 1}^r \bs{\beta}(\bl{b}_i)\bl{a}_i \in \mathrm{cone}(\bl{a}_1,\ldots,\bl{a}_r)$ whenever $\bs{\beta} \in \mc{C}(\mc{P}_2)^*$. The number $r$ is at least equal to $s$ by assumption. 
\end{proof}

The upper bound in Lemma \ref{lem:NDrankUpperBound} is exact when one of the constituent order cones is simplicial. The posets occurring in many applications will satisfy this condition; the selenium example in Section \ref{sec:NonDecRankIntro} satisfies this condition and has a maximum rank of four.

\begin{theorem}
\label{thm:MaxRankSimplicial}
        If $\mc{C}(\mc{P}_j)$ has $q_j$ extremal rays, $j = 1,2$, and $\mc{C}(\mc{P}_2)$ is simplicial then the maximum ND rank of a matrix in $\mc{N}_{< \infty}$ is $\min(q_1,q_2)$. 
\end{theorem}
\begin{proof}
    Without loss of generality it may be assumed that $\mc{C}(\mc{P}_2) = \mb{R}^{q_2}_+$ by the affine-invariance of the ND rank (Theorem \ref{thm:NDNonnegativeFactEquivalence}). As $\mc{C}(\mc{P}_2)$ is simplicial $p_2 = q_2$. Let $\bl{v}_1,\ldots,\bl{v}_{q_1}$ be the extremal rays of $\mc{C}(\mc{P}_1)$. If $q_2 \leq p_1 \leq q_1$ then any full rank matrix $\bl{T}$ will have $q_2 = \text{rank}(\bl{T}) \leq \text{NDrank}(\bl{T}) \leq q_2$ by Lemma \ref{lem:NDrankUpperBound}. If $p_1 \leq q_2 \leq q_1$ take $\bl{T} = \sum_{i = 1}^{q_2} \bl{v}_i \otimes \bl{e}_i$. Then $\bl{T}(\mc{C}(\mc{P}_2)^*) = \bl{T}(\mb{R}^{q_2}_+) = \text{cone}(\bl{v}_1,\ldots,\bl{v}_{q_2})$. If $ \text{cone}(\bl{v}_1,\ldots,\bl{v}_{q_2}) \subseteq \text{cone}(\bl{a}_1,\ldots,\bl{a}_t)$ for some $\bl{a}_i \in \mc{C}(\mc{P}_1)$ then vectors proportional to $\bl{v}_1,\ldots,\bl{v}_{q_2}$ must appear in the various $\bl{a}_i$s as the $\bl{v}_i$s are extremal. Using Lemma \ref{lem:MaxRankDualSpaceImageLowBound} the ND rank of $\bl{T}$ is $q_2$. Finally, when $p_1 \leq q_1 \leq q_2$, taking $\bl{T} = \sum_{i = 1}^{q_1} \bl{v}_i \otimes \bl{e}_i$, a similar argument using $\bl{T}(\mc{C}(\mc{P}_2)^*) = \bl{T}(\mb{R}^{q_2}_+) = \text{cone}(\bl{v}_1,\ldots,\bl{v}_{q_1}) = \mc{C}(\mc{P}_1)$ shows that the ND rank of $\bl{T}$ is $q_1$. 
\end{proof}

Of note in the above theorem is that the maximum ND rank of a matrix in $\mb{R}^{p_1 \times p_2}$ can potentially be much larger than the usual maximum matrix rank of $\min(p_1,p_2)$. 

The prototypical example of a poset with a non-simplicial order cone is a collider (Definition \ref{def:Collider}). The next theorem shows that when the column and row posets of a matrix both have the form of a collider the maximum ND rank can be larger than even $\max(p_1,p_2)$.

\begin{theorem}
\label{thm:MaxRankCollider}
 Let $\mc{P} = \{x_1,\ldots,x_{p}\}$ where $x_i \prec x_{p}$ for all $i < p$. The maximum rank in $\mc{C}(\mc{P}) \otimes \mc{C}(\mc{P}) \subset \mb{R}^{p \times p}$ is at least $2(p-1)$. When $p = 3$ the maximum rank is equal to $4$.
\end{theorem}
\begin{proof}
Let $p \geq 3$ so that $\mc{P}$ contains a collider. The extremal rays $\bl{v}_1,\ldots,\bl{v}_{2^{p-1}}$, in $\mc{C}(\mc{P})$ can be taken to have the form $(\bl{w},1)$ where $\bl{w} \in \mb{R}^{p-1}$ is any one of the $2^{p-1}$ vectors that have entries equal to either zero or one. Hence, $\mc{C}(\mc{P})$ is equal to the homogenization \citep[Sec 1.5]{ZieglerPolytopes} of a hypercube.  The facets of $\mc{C}(\mc{P})$ each contain $2^{p-2}$ extremal rays and there are $2(p-1)$ facets; one facet for each equality $0 = f_{x_i}$, and one facet for $f_{x_i} = f_{x_p}$, where $i < p$. The facet $\{\bl{f}: f_{x_i} = 0\} \cap \mc{C}(\mc{P})$ contains the extremal rays of the form $(\bl{w},1)$ with $\bl{w}_i = 0$, while the facet $\{\bl{f}:f_{x_i} = f_{x_p}\} \cap \mc{C}(\mc{P})$ contains the remaining extremal rays $(\bl{w},1)$ with $w_i = 1$. These two facets are opposite to each other and partition the set of extremal rays.

The goal of the remainder of the proof to construct a matrix $\bl{T}$ that has an ND rank bounded below by $2(p-1)$. Choose a facet $F$ of $\mc{C}(\mc{P})$ that has $F'$ as its opposite facet. Let $J$ be an index set of size less than $p-1$ where $\bl{v}_j \in F$ for every $j \in J$. The set $\text{cone}(\bl{v}_j: j \in J)$ has a dimension that is less than $p-1$. If we remove all such cones from $F$ we are left with a non-empty set
\begin{align*}
    F \;\backslash \bigg(\bigcup_{J: \vert J \vert < p-1,\; \forall j \in J \;\bl{v}_j \in F}  \text{cone}(\bl{v}_j: j \in J)\bigg),
\end{align*}
because $F$ has dimension $p-1$. Choose a point $\sum_{i \in I_F} \alpha_i \bl{v}_i$ in the above set, where the index subset has $\vert I_F \vert = p-1$ and $\bl{v}_i \in F$ for every $i \in I_F$. Repeat the above process for the facet $F'$ to obtain an analogous point $\sum_{i \in I_{F'}}\alpha_i \bl{v}_i$. We define the matrix $\bl{T} = \sum_{i \in I} \alpha_i \bl{v}_i \otimes \bl{v}_i$ where $I = I_F \cup I_{F'}$.

Assume that $\bl{T}  = \sum_{i = 1}^r \bl{a}_i \otimes \bl{b}_i$ is a minimum ND rank decomposition, with $\bl{a}_i = \sum_{j = 1}^{2^{p-1}} \lambda_{ij} \bl{v}_j$ and $\bl{b}_i = \sum_{l = 1}^{2^{p-1}} \mu_{il} \bl{v}_l$, $\lambda_{ij}, \mu_{il} \geq 0$. For any pair of extremal rays $\bl{v}_l \neq \bl{v}_j$ choose a facet $G$ with corresponding dual vector $\bl{h}$ so that $\bl{v}_l \in G$ but $\bl{v}_j \notin G$. Let $\bl{h}'$ be the dual vector corresponding to the facet $G'$ opposite of $G$. We compute
\begin{align*}
   0 = \bl{h}^\intercal\sum_{i \in I} \alpha_i (\bl{v}_i \otimes\bl{v}_i) \bl{h}' = \bl{h}^\intercal \big(\sum_{i = 1}^r \sum_{j = 1}^{2^{p-1}} \sum_{l = 1}^{2^{p-1}} \lambda_{ij}\mu_{il} (\bl{v}_j \otimes \bl{v}_l)\big) \bl{h}' \geq \sum_{i = 1}^r \lambda_{ij}\mu_{il} (\bl{h}^\intercal \bl{v}_j)(\bl{v}_l^\intercal \bl{h}').
\end{align*}
The first equality follows because every extremal ray in $\mc{C}(\mc{P})$ must be in only one of the facets $G$ or $G'$. The inequality is a result of the middle expression being a sum of nonnegative terms. 
By construction $(\bl{h}^\intercal \bl{v}_j)(\bl{v}_l^\intercal \bl{h}') > 0$, which implies that $\lambda_{ij}\mu_{il} = 0$ whenever $j \neq l$. We conclude that $\bl{a}_i \otimes \bl{b}_i \, \propto \, \bl{v}_{j_i} \otimes \bl{v}_{j_i}$ for some index $j_i$, with $\bl{T} = \sum_{i = 1}^r c_{i} \bl{v}_{j_i} \otimes \bl{v}_{j_i}$ and $c_i > 0$. Let $\bl{h}$ be a dual vector corresponding to the hyperplane that supports the facet $F'$. Then
\begin{align*}
   \sum_{i \in I_F}\alpha_i \bl{v}_i =  \sum_{i \in I} \alpha_i (\bl{v}_i \otimes \bl{v}_i) (\bl{h}) = \sum_{i = 1}^r c_i (\bl{v}_{j_i} \otimes \bl{v}_{j_i}) (\bl{h}) = \sum_{i: \bl{v}_{j_i} \in F} c_i \bl{v}_{j_i},
\end{align*}
where without loss of generality it is assumed that $\bl{h}$ is scaled so that $\bl{h}^\intercal \bl{v}_i = 1$ for $\bl{v}_i \in F$. By how the vector $ \sum_{i \in I_F}\alpha_i \bl{v}_i$ was constructed, there must exist at least $p-1$ terms in the sum $ \sum_{i: \bl{v}_{j_i} \in F} c_i \bl{v}_{j_i}$. A symmetric argument shows that the set of indices $\{i: \bl{v}_{j_i} \in F'\}$ has a size of at least $p-1$. As the facets $F$ and $F'$ do not share any extremal rays, we conclude that the sum $\sum_{i = 1}^r c_i \bl{v}_{j_i} \otimes \bl{v}_{j_i}$ must include at least $2(p-1)$ terms and so $r \geq 2(p-1)$.   

In the case when $p = 3$, Lemma \ref{lem:NDrankUpperBound} shows that the maximum rank is also at most $4 = 2(3-1)$, which is the number of vertices in a two-dimensional cube.
\end{proof}

The problem of finding the maximum rank of tensor decompositions for arbitrary order cones and for more general polyhedral cones remains an open question.

\subsection{A Matrix Tri-Factorization Formulation}
A succinct representation of a rank-$r$, nonnegative matrix factorization of $\bl{T} \in \mb{R}^{p_1 \times p_2}_+$ is $\bl{T}  = \bl{A}^{(1)}(\bl{A}^{(2)})^\intercal$ where $\bl{A}^{(j)} \in \mb{R}_+^{p_j \times r}$. If $\bl{T} = \sum_{i = 1}^r \bl{a}^{(i1)} \otimes \bl{a}_2^{(i2)}$ is an ND factorization we can likewise write this as $\bl{T} = \bl{A}^{(1)}(\bl{A}^{(2)})^\intercal$ where $\bl{A}^{(j)} \in \mb{R}^{p_j \times r}$ has columns that are given by the $\bl{a}^{(ij)}$s for $i = 1,\ldots,r$. Let $\bl{V}^{(j)} \in \mb{R}^{p_i \times q_j}$ be a matrix that has columns that are equal to the $q_j$ extremal rays in $\mc{C}(\mc{P}_j)$. As each $\bl{a}^{(ij)}$ is in the conical hull of the columns of $\bl{V}^{(j)}$ there exist nonnegative matrices $\bl{H}^{(i)} \in \mb{R}_+^{q_i \times r}$ where $\bl{A}^{(j)} = \bl{V}^{(j)} \bl{H}^{(j)}$ for $j = 1,2$. Defining $\bl{H} = \bl{H}^{(1)} (\bl{H}^{(2)})^\intercal \in \mb{R}^{q_1 \times q_2}_+$ the ND factorization can be written as
\begin{align*}
    \bl{T} = \bl{V}^{(1)} \bl{H}^{(1)} (\bl{H}^{(2)})^\intercal (\bl{V}^{(2)})^\intercal = \bl{V}^{(1)} \bl{H} (\bl{V}^{(2)})^\intercal.
\end{align*}
This leads to an observation connecting nondecreasing and nonnegative ranks:
\begin{lemma}
    The nondecreasing rank of a matrix $\bl{T} \in \mc{N}_{< \infty}$ is the smallest nonnegative rank of a matrix $\bl{H} \in \mb{R}^{q_1 \times q_2}$ that satisfies the equation $\bl{T} = \bl{V}^{(1)} \bl{H} (\bl{V}^{(2)})^\intercal$, where $\bl{V}^{(j)} \in \mb{R}^{p_j \times q_j}$ are fixed matrices with columns that are equal to the $q_j$ extremal rays of $\mc{C}(\mc{P}_j)$. 
\end{lemma}
The above tri-factorization formulation is related to nonnegative matrix tri-factorizations \citep{TriNMFwang2011fast}, as well as convex NMF \citep{JordanDing2008convexSemi}, which specifies a fixed dictionary of vectors to be used in the factorization. 

As an example of this factorization, the matrix attaining the maximum rank appearing in the proof of Theorem \ref{thm:MaxRankCollider} for $p = 3$ has the following representation
\begin{align*}
\begin{bmatrix}
    2 & 1 & 2
    \\
    1 & 2 & 2
    \\
    2 & 2 & 4
\end{bmatrix}
= \begin{bmatrix}
    0 & 1 & 0 & 1 
    \\
    0 & 0 & 1 & 1
    \\
    1 & 1 & 1 & 1
\end{bmatrix} \begin{bmatrix}
        1 & 0 & 0 & 0
        \\
        0 & 1 & 0 & 0 
        \\
        0 & 0 & 1 & 0
        \\
        0 & 0 & 0 & 1
    \end{bmatrix} \begin{bmatrix}
        0 & 0 & 1
        \\
        1 & 0 & 1
        \\
        0 & 1 & 1
        \\
        1 & 1 & 1
    \end{bmatrix} = \bl{V}^{(1)} \bl{H} (\bl{V}^{(2)})^\intercal,
\end{align*}
where $\bl{H}$ has a nonnegative rank of four and the rows of $\bl{V}^{(1)}  = \bl{V}^{(2)}$ are equal to the four extremal rays of the collider $\mc{C}(\mc{P}_j)$ outlined in Lemma \ref{lem:OrderConeExtremalRays}.

\subsection{Typical ND Ranks}
A typical nondecreasing rank is defined as a number $r$ such that $\mc{N}_{r}$ has a non-empty interior. The probabilistic interpretation of an typical ND rank is that if a matrix $\bl{X}$ is drawn from a distribution supported on $\mc{N}_{< \infty}$ that has a density with respect to the Lebesgue measure then $\text{Pr}(\text{NDrank}(\bl{X}) = r) > 0$ when $r$ is a typical rank. In this section we determine the typical ranks for the two matrix settings outlined in Theorems \ref{thm:MaxRankSimplicial} and \ref{thm:MaxRankCollider}, as well as for any other setting where the maximum ND rank is known.
\begin{theorem}
\label{thm:MaxRankisTypical}
    The maximum ND rank is always a typical ND rank. 
\end{theorem}
\begin{proof}
Let $m$ be the maximum ND rank and choose a $\bl{T}$ with $\text{NDrank}(\bl{T}) = m$. Let $B_{n^{-1}}(\bl{T})$ be an open ball of radius $n^{-1}$ centered at $\bl{T}$. The set $B_{n^{-1}}(\bl{T}) \cap \mc{N}_{< \infty}$ has a non-empty interior. If $\bl{T} \in \text{int}(\mc{N}_{< \infty})$ this is immediate. Otherwise, if $\bl{T}$ is in the boundary of $\mc{N}_{< \infty}$ take an $\bl{R} \in \text{int}(\mc{N}_{< \infty})$, which exists because $\mc{N}_{< \infty}$ is proper, and consider the line segment $\lambda\bl{T} + (1-\lambda) \bl{R}$, $\lambda \in [0,1]$. For all $\lambda \in [0,1)$ the point $\lambda\bl{T} + (1-\lambda) \bl{R}$ is in $\text{int}(\mc{N}_{< \infty})$, and for large enough $\lambda$ this point is also in $B_{n^{-1}}(\bl{T})$. this provides a point in $B_{n^{-1}}(\bl{T}) \cap \text{int}(\mc{N}_{< \infty})$.  
Now if the maximum ND rank was not a typical rank, the set $B_{n^{-1}}(\bl{T}) \cap \mc{N}_{< \infty}$, that has a non-empty interior, would contain an $\bl{S}_n \in \mc{N}_{\leq m-1}$. Taking $n \rightarrow \infty$ gives a sequence in $\mc{N}_{\leq m-1}$ with $\bl{S}_n \rightarrow \bl{T}$. As the set $\mc{N}_{\leq m-1}$ is closed (see Theorem \ref{thm:BorderRank}) this is a contradiction to $\bl{T} \in \mc{N}_m$. Note that this proof also shows that the set $\mc{N}_m$ is an open set in the subspace topology of $\mc{N}_{< \infty}$ since there must exist an $n$ with $B_{n^{-1}}(\bl{T}) \cap \mc{N}_{< \infty} \subset \mc{N}_m$.  
\end{proof}

The next theorem extends results that have appeared in \citep{ComonXrank,sumi2018maximal} for the nonnegative rank and completely characterizes all possible typical ND ranks. A typical real rank in $\otimes_{j = 1}^k \mb{R}^{p_j}$ refers to any $r$ where the set of real, rank-$r$ tensors in $\otimes_{j = 1}^k \mb{R}^{p_j}$ has a non-empty interior.  

\begin{theorem}
\label{thm:TypicalRankBoundsGeneric}
    Let $r_{Re}$ be the minimum, typical, real rank for tensors in $\mb{R}^{p_1 \times \cdots \times p_k}$. Any $r$ with $r_{Re} \leq r \leq \mathrm{maxNDrank}$ is a typical ND rank. 
\end{theorem}
\begin{proof}
    We adapt the proof provided in Theorem 3.8 of \citep{sumi2018maximal} by replacing all instances of the the cone of nonnegative vectors $\mb{R}^{p_j}_{+}$ with the order cone $\mc{C}(\mc{P}_j)$. This proof does not rely on any of the structure of the cone $\mb{R}^{p_j}_{+}$; the only place in this proof where the cone comes into play is in showing that the maximum rank is a typical nonnegative rank. In the ND rank case the analogous result is shown in Theorem \ref{thm:MaxRankisTypical}. Theorem 3.8 also depends on Lemmas 3.4-3.7 in \citep{sumi2018maximal}. Similarly, none of these lemmas rely on the specific structure of $\mb{R}_+^{p_i}$ and also apply to $\mc{C}(\mc{P}_j)$ and the ND rank.
\end{proof}
 When considering rank-$r$ factorizations over the complex numbers there is only one typical rank, also referred to as the generic rank \citep{blekhermanteitlerranks}. In general, there can be more than one typical, real or nonnegative rank \citep{sumi2015typical}. However, for matrices, since the maximum rank of a $p_1 \times p_2$ matrix is equal the minimum typical rank of $\min(p_1,p_2)$, this is the only typical, real or nonnegative rank. As the maximum ND rank can be larger than $\min(p_1,p_2)$ there can be more than one typical ND rank in the matrix setting, as is illustrated by the following corollaries. 

\begin{corollary}
\label{cor:TypicalRankSimplicial}
        If $\mc{C}(\mc{P}_j) \subset \mb{R}^{p_j}$ has $q_j$ extremal rays, $j = 1,2$, and $\mc{C}(\mc{P}_2)$ is simplicial then $r$ is a typical ND rank in $\mc{C}(\mc{P}_1) \otimes \mc{C}(\mc{P}_2)$ if and only if $\min(p_1,p_2) \leq r \leq \min(q_1,q_2)$. 
\end{corollary}
\begin{proof}
    This is a direct consequence of Theorems \ref{thm:MaxRankSimplicial}, \ref{thm:MaxRankisTypical}, and \ref{thm:TypicalRankBoundsGeneric}; collectively these results show that any rank greater than or equal to $\min(p_1,p_2)$ and less than or equal to the  maximum rank of $\min(q_1,q_2)$ is a typical rank. 
\end{proof}

\begin{corollary}
     If $\mc{P} = \{x_1,\ldots,x_{p}\}$ where $x_i \prec x_{p}$ for all $i < p$ then $r$ is a typical rank in  $\mc{C}(\mc{P}) \otimes \mc{C}(\mc{P})$ if $p \leq r \leq 2(p-1)$. 
\end{corollary}
\begin{proof}
The proof is the same as in Corollary \ref{cor:TypicalRankSimplicial} except that the maximum rank is now bounded below by $2(p-1)$ in Theorem \ref{thm:MaxRankCollider}.  
\end{proof}

\section{ND Rank One and Two}
\label{sec:NDRankOneTwo}
The relationship between rank and nonnegative rank is especially simple when the rank is either one or two: the inequality $\text{rank}(\bl{T}) \leq \text{rank}_+(\bl{T})$ is satisfied with equality. In the next two results we show that this property can be extended to the nondecreasing rank.

\begin{theorem}
    If $\bl{T} \in \mc{C}(\times_{j = 1}^k \mc{P}_j)$ is a rank-one tensor, then $\bl{T}$ has an ND rank of one.
\end{theorem}
\begin{proof}
By assumption $\bl{T} = \otimes_{j = 1}^k \bl{v}^{(j)} \neq \bl{0}$, where it can also be assumed that every $\bl{v}^{(j)} \in \mb{R}^{p_j}_+$ as otherwise $\bl{T}$ would have a negative entry. Suppose for a contradiction that $\bl{v}^{(i)}$ is not monotone so that there exists $x_i \prec y_i$ in $\mc{P}_i$ with $v^{(i)}_{x_i} > v^{(i)}_{y_i}$. Choosing $z_j \in \mc{P}_j$ with $v^{(j)}_{z_j} > 0$ for $j \neq i$ we find that 
\begin{align*}
    T_{z_1 \ldots z_{i-1} y_i z_{i+1} \ldots z_k } -     T_{z_1 \ldots z_{i-1} x_i z_{i+1} \ldots z_k } = (v_{y_i}^{(i)} - v_{x_i}^{(i)}) \prod_{j \neq i} v^{(j)}_{z_j} < 0,
\end{align*}
contradicting $\bl{T} \in  \mc{C}(\times_{i = 1}^k \mc{P}_i)$. 
\end{proof}

\begin{theorem}
\label{thm:NDRankTwo}
    If $\bl{T} \in \mc{N}_{< \infty}$ is a rank-two matrix then $\bl{T}$ also has an ND rank of two.  
\end{theorem}
\begin{proof}
As $\bl{T} \in \mc{N}_{< \infty}$ we have that $\bl{T} = \sum_{i = 1}^r \bl{a}^{(i)} \otimes \bl{b}^{(i)}$, $\bl{a}^{(i)} \in \mc{C}(\mc{P}_1)$, $\bl{b}^{(i)} \in \mc{C}(\mc{P}_2)$, and $\bl{T} \bl{h} = \sum_{i = 1}^r (\bl{h}^\intercal \bl{b}^{(i)}) \bl{a}^{(i)} \in \mc{C}(\mc{P}_1)$ for all $\bl{h} \in \mc{C}(\mc{P}_2)^*$. Similarly, $\bl{h}^\intercal \bl{T} \in \mc{C}(\mc{P}_2)$ for all $\bl{h} \in \mc{C}(\mc{P}_1)^*$. As $\mc{C}(\mc{P}_2)^*$ is full-dimensional \citep{TensorProdConeMulansky} $\bl{T}(\mc{C}(\mc{P}_2)^*) \subseteq \mc{C}(\mc{P}_1)$ is a two-dimensional cone. Let $\bl{v}^{(1)},\bl{v}^{(2)}$ be the extremal rays of the two-dimensional cone $\text{col}(\bl{T}) \cap \mc{C}(\mc{P}_1)$. There exist $\bl{w}^{(1)},\bl{w}^{(2)}$ such that $\bl{T} = \bl{v}^{(1)} \otimes \bl{w}^{(1)} +  \bl{v}^{(2)} \otimes \bl{w}^{(2)}$ and it remains to show that $\bl{w}^{(i)} \in \mc{C}(\mc{P}_2)$ for $i = 1,2$.  There must exist a face of $\mc{C}(\mc{P}_1)$ that contains $\bl{v}^{(1)}$ but not $\bl{v}^{(2)}$, since otherwise $\lambda \bl{v}^{(1)} + (1 - \lambda) \bl{v}^{(2)}$ would be contained in $\mc{C}(\mc{P}_1)$ for a small enough $\lambda > 1$, contradicting the assumption that $\bl{v}^{(1)}$ was extremal in $\text{col}(\bl{T}) \cap \mc{C}(\mc{P}_1)$. Let $\bl{h} \in \mc{C}(\mc{P}_1)^*$ be a dual vector corresponding to the hyperplane that supports this face. Then $\bl{h}^\intercal \bl{T} = (\bl{h}^\intercal \bl{v}^{(1)}) \bl{w}^{(1)} + (\bl{h}^{\intercal} \bl{v}^{(2)}) \bl{w}^{(2)} = (\bl{h}^{\intercal} \bl{v}^{(2)}) \bl{w}^{(2)} \in \mc{C}(\mc{P}_2)$, implying $\bl{w}^{(2)} \in \mc{C}(\mc{P}_2)$ because $\bl{h}^\intercal \bl{v}^{(2)} > 0$. The same argument shows that $\bl{w}^{(1)} \in \mc{C}(\mc{P}_2)$. 
\end{proof}

A ramification of the above theorems is that defining equations for the semialgebraic sets \citep[Def 2.1.4]{RealAlgGeoBochnakCosteRoy} of ND rank-one tensors and ND rank-two matrices are known. In the former case, a tensor $\bl{T}$ is in $\mc{N}_{\leq 1}$ if and only if it is monotonic and if it satisfies the determinantal equations for a rank-one tensor \citep[Sec 3.4.1]{TensorsLandsberg}. In the latter case, a matrix $\bl{T}$ is in $\mc{N}_{\leq 2}$ if and only if every $3 \times 3$ minor of $\bl{T}$ vanishes and if $\bl{T}$ satisfies the inequalities for the $\mc{H}$-representation of $\mc{N}_{< \infty}$ described in Sections \ref{sec:OrderConeGeomExisenceofFactor} and \ref{sec:EquivofNDandNNRanks}. Semialgebraic conditions under which a tensor of rank two also has a nonnegative rank of two are put forward in \citep{AllmanZwiernikRankTwo}. Whenever all order cones $\mc{C}(\mc{P}_j)$ are simplicial such conditions can be easily translated, via Theorem \ref{thm:NDNonnegativeFactEquivalence}, into conditions for a tensor to have an ND rank of two.

\section{ND Border Rank Equals ND Rank}
\label{sec:NDBorderRank}
Due to the presence of noise, in many applications it is unlikely that a data tensor has an exact, low-rank, ND factorization. Instead a solution to an optimization problem of the form $\argmin_{\bs{\theta} \in \mc{N}_{\leq r}}D(\bl{T}, \bs{\theta})$ is sought, where $D$ is a divergence that measures the discrepancy between the observed data tensor $\bl{T}$ and an approximating, low ND rank tensor $\bs{\theta}$. To ensure that a solution to this optimization problem exists it is important that the set $\mc{N}_{\leq r}$ be closed. The set of tensors with real tensor rank at most $r$ is not closed, which necessitates the introduction of the concept of border rank \citep[Sec 2.4.5]{TensorsLandsberg}. It is shown in this section, extending the proof for nonnegative tensor ranks provided in \citep{BorderRanklim2009nonnegative}, that $\mc{N}_{\leq r}$ is closed and there is no need for the additional notion of border rank.

\begin{theorem}[\citep{BorderRanklim2009nonnegative} Thm 6.1]
\label{thm:BorderRank}
The set $\mc{N}_{\leq r}$ is closed. Equivalently, for any $\bl{T} \in \mb{R}^{p_1 \times \cdots \times p_k}$ there exists a solution to the optimization problem $\inf_{\bs{\theta} \in \mc{N}_{\leq r}} \Vert \bl{T} - \bs{\theta} \Vert_F$, where $\Vert \cdot \Vert_F$ is the Frobenius norm. 
\end{theorem}
\begin{proof}
    Assume that $\bs{\theta}^{(n)} = \sum_{i = 1}^r \lambda_i^{(n)} \otimes_{j = 1}^k  \bl{v}_{ij}^{(n)}$ is a sequence of tensors in $\mc{N}_{\leq r}$ with $\Vert \bl{T} - \bs{\theta}^{(n)} \Vert_F \rightarrow \inf_{\bs{\theta} \in \mc{N}_{\leq r}} \Vert \bl{T} - \bs{\theta} \Vert_F$. We will first show that $\bs{\theta}^{(n)}$ has a convergent subsequence that converges to a point in $\mc{N}_{\leq r}$. By the continuity of the objective function this point of convergence will be a solution to the given optimization problem. The $\bl{v}_{ij}^{(n)} \in \mc{C}(\mc{P}_j)$ are assumed, without loss of generality, to be scaled so that $\Vert \bl{v}_{ij}^{(n)}\Vert_2 = 1$ and $\lambda_i^{(n)} \geq 0$. It is claimed that there exists an $M$ such that $\sup_n \lambda_i^{(n)} \leq M$ for all $i$. To show this, note that $\Vert \bl{T} -  \bs{\theta}^{(n)} \Vert_F \geq \Vert \bs{\theta}^{(n)} \Vert_F - \Vert \bl{T} \Vert_F$ and
\begin{align*}
     \Vert \bs{\theta}^{(n)} \Vert_F^2 &= \sum_{i = 1}^r \sum_{i' = 1}^r \lambda_i^{(n)} \lambda_{i'}^{(n)} \langle \otimes_{j = 1}^k \bl{v}_{ij}^{(n)},  \otimes_{j = 1}^k \bl{v}_{i'j}^{(n)} \rangle 
     \\
     &=  \sum_{i = 1}^r \sum_{i' = 1}^r \lambda_i^{(n)} \lambda_{i'}^{(n)} \prod_{j = 1}^k \langle  \bl{v}_{ij}^{(n)}, \bl{v}_{i'j}^{(n)} \rangle \geq (\lambda_i^{(n)})^2 \prod_{j = 1}^k \Vert \bl{v}_{ij}^{(n)} \Vert_2^2 = (\lambda_i^{(n)})^2,
\end{align*}
where the last inequality follows because $\langle  \bl{v}_{ij}^{(n)}, \bl{v}_{i'j}^{(n)} \rangle \geq 0$ for every $i,i',j$ as $\mc{C}(\mc{P}_j) \subseteq \mb{R}_
+^{p_j}$. If $\lambda_i^{(n)} \rightarrow \infty$ for any $i$ then $\Vert \bl{T} - \bs{\theta}^{(n)} \Vert_F \rightarrow \infty$, a contradiction. It follows that there exists a convergent subsequence with $\lambda_i^{(n_m)} \rightarrow \lambda_i$ and $\bl{v}_{ij}^{(n_m)} \rightarrow \bl{v}_{ij} \in \mc{C}(\mc{P}_j)$, with the latter inclusion following from $\mc{C}(\mc{P}_j)$ being closed. Thus, $\bs{\theta}^{(n_m)} \rightarrow \tilde{\bs{\theta}} = \sum_{i = 1}^r \lambda_i \otimes_{j = 1}^k \bl{v}_{ij} \in \mc{N}_{\leq r}$ is a convergent subsequence. The point of convergence $\tilde{\bs{\theta}} = \argmin_{\bs{\theta} \in \mc{N}_{\leq r}} \Vert \bl{T} - \bs{\theta}\Vert_F$ is a solution to the optimization problem of interest. To show that $\mc{N}_{\leq r}$ is closed, first observe that if $\bl{T}$ is in the closure of $\mc{N}_{\leq r}$, then $\inf_{\bs{\theta} \in \mc{N}_{\leq r}} \Vert \bl{T} - \bs{\theta} \Vert_F = 0$. However, the point $\tilde{\theta} \in \mc{N}_{\leq r}$, obtained from a convergent subsequence as described above, is a solution to the minimization problem, implying that $\Vert \bl{T} - \tilde{\bs{\theta}}\Vert_F = \inf_{\bs{\theta} \in \mc{N}_{\leq r}} \Vert \bl{T} - \bs{\theta} \Vert_F = 0$. That is, $\bl{T} = \tilde{\bs{\theta}} \in \mc{N}_{\leq r}$. 
\end{proof}

\section{Finding Low ND Rank Approximations}
\label{sec:Optimization}
Solving the optimization problem $\inf_{\bs{\theta} \in \mc{N}_{\leq r}} \Vert \bl{T} - \bs{\theta} \Vert_F$ introduced in the previous section is a one possibility for finding a low ND rank approximation to a data tensor $\bl{T}$. However, when the entries of $\bl{T}$ consist of count data or positive data the Frobenius norm objective function may not be the most appropriate criteria to minimize. We examine a likelihood-based approach for estimating the mean of $\bl{T}$ in this section. It is assumed that the entries $T_{i_1\ldots i_k}$ are independently sampled from an exponential family of distributions with mean parameter \citep[Ch 3]{BrownExponentialFamilies} $E(T_{i_1\ldots i_k}) = \theta_{i_1\ldots i_k}$ for every $(i_1,\ldots,i_k) \in \times_{j = 1}^k [p_j]$, where $E(\cdot)$ is the expected value. The tensor $\bs{\theta}$ is constrained to have an ND rank of at most $r$ and is estimated by maximizing the observed likelihood
\begin{align*}
    \hat{\bs{\theta}} = \underset{\bs{\theta} \in \mc{N}_{\leq r}}{\argmax} \; p(\bl{T}|\bs{\theta}) = \underset{\bs{\theta} \in \mc{N}_{\leq r}}{\argmax} \; \sum_{i_1 = 1}^{p_1}\cdots\sum_{i_k = 1}^{p_k} \log\big(p(T_{i_1\ldots i_k}|\theta_{i_1\ldots i_k})\big).
\end{align*}
The function $p(\bl{T}|\bs{\theta})$ is the joint probability distribution or probability mass function of the tensor $\bl{T}$. Different distributional assumptions about $T_{i_1\ldots i_k}$ lead to different choices of $p(T_{i_1\ldots i_k}|\theta_{i_1\ldots i_k})$ and hence different optimization problems. Of primary interest are the following standard distributions:
\begin{enumerate}
    \item If $T_{i_1\ldots i_k} \in \mb{R}$ then we may assume that $T_{i_1\ldots i_k} \sim \mc{N}(\theta_{i_1\ldots i_k},\sigma^2)$ are independent Gaussians. The corresponding optimization problem is
    \begin{align*}
\underset{\bs{\theta} \in \mc{N}_{\leq r}}{\argmin} \; \Vert \bl{T} - \bs{\theta} \Vert_F^2.
    \end{align*}
    Note that the value of the variance $\sigma^2$ does not change the above optimization problem. 
\item If the entries of $\bl{T}$ are counts in $\mb{N}$ that sum to $n$ then we may assume that $\bl{T} \sim \text{Multinomial}(\bs{\theta},n)$. Note that in this case the entries of $\bl{T}$ are not independent. The corresponding optimization problem over the probability simplex $\Delta_{\prod_i p_i - 1} \subset \mb{R}^{p_1 \times \cdots p_k}$ is 
\begin{align*}
\underset{\bs{\theta} \in \mc{N}_{\leq r} \cap \Delta_{\prod_i p_i - 1} }{\argmin} \; -\sum_{i_1 = 1}^{p_1}\cdots\sum_{i_k = 1}^{p_k} T_{i_1\ldots i_k} \log(\theta_{i_1\ldots i_k}).
\end{align*}
\item If the $T_{i_1\ldots i_k} \in \mb{N}$ are counts that do not necessarily have to sum to $n$ we may assume that $T_{i_1\ldots i_k} \sim \text{Poisson}(\theta_{i_1\ldots i_k})$ independently.  The corresponding optimization problem is 
\begin{align*}
  \underset{\bs{\theta} \in \mc{N}_{\leq r}}{\argmin} \; \sum_{i_1 = 1}^{p_1}\cdots\sum_{i_k = 1}^{p_k} \big(\theta_{i_1 \ldots i_k} -T_{i_1\ldots i_k} \log(\theta_{i_1\ldots i_k})\big).
\end{align*}
In both this problem and the multinomial problem whenever $T_{i_1\ldots i_k} = 0$ the respective $T_{i_1\ldots i_k}\log(\theta_{i_1\ldots i_k})$ term does not appear in the objective function. 
\item If $T_{i_1 \ldots i_k} \in (0,\infty)$ are positive then we may assume that $T_{i_1\ldots i_k} \sim \text{Exponential}(\theta_{i_1\ldots i_k}^{-1})$ independently.  The corresponding optimization problem is
\begin{align*}
\underset{\bs{\theta} \in \mc{N}_{\leq r}}{\argmin} \;  \sum_{i_1 = 1}^{p_1}\cdots\sum_{i_k = 1}^{p_k} \big( \log(\theta_{i_1\ldots i_k}) + \tfrac{T_{i_1\ldots i_k}}{\theta_{i_1 \ldots i_k}}\big).
\end{align*}
\end{enumerate}
These four optimization problems represent special cases of $\beta$-divergences, that are commonly applied to NMF problems \citep{BetaDivergenceNMF}, where $\beta$ is respectively equal to $2$, $1$, and $0$ under the distributional assumptions in $1.$ $3.$ and $4.$ above. Similar to Proposition 7.2 in \citep{BorderRanklim2009nonnegative}, which shows that optimal nonnegative approximations exist for Bregman divergences, the next result shows that this also holds for nondecreasing factorizations, but does not require any additional constraints on $\bs{\theta}$.

\begin{lemma}
\label{lem:ExistenceofDivergenceOptimizer}
    If the entries of $\bl{T}$ are in the respective subsets $\mb{R},\mb{N},\mb{N},(0,\infty)$ in each of the four optimization problems above then there exists a minimizer in each of the problems. 
\end{lemma}
\begin{proof}
    This follows for $1.$ by Theorem \ref{thm:BorderRank}. For the remaining three problems it is seen that there exists an $M$ where any minimizer must be contained in the region $\mc{S} = \{\bs{\theta}:\theta_{i_1\ldots i_k} \leq M, \forall i_1,\ldots,i_k\}$, as the objective functions in $3.$ and $4.$ diverge to $\infty$ whenever $\theta_{i_1\ldots i_k} \rightarrow \infty$. In case $2.$ $\bs{\theta}$ is trivially contained in such an $\mc{S}$ because the probability simplex is compact. As $\mc{S} \cap \mc{N}_{\leq r}$ is compact by Theorem \ref{thm:BorderRank} and the objective functions are continuous on this region, existence of the solutions follows. Here the objective functions are extended continuously to take values in $\mb{R} \cup \{\infty\}$ with $-T_{i_1\ldots i_k} \log(0) \coloneq \infty$ and $T_{i_1\ldots i_k}/0 \coloneq \infty$ when $T_{i_1\ldots i_k} \neq 0$.  
\end{proof}

Before examining a least squares optimization algorithm for finding rank-$r$ ND approximations we discuss a couple of results in low-rank cases where finding an ND approximation is the same as simply performing a singular value decomposition.

The Perron-Frobenius theorem \citep[Thm 3.2]{NonnegativeMatricesBookBerman} that states that nonnegative, square matrices have nonnegative leading eigenvectors, has been extended to nonnegative, cubical tensors in \citep[Thm 1]{Lim2005singularPerronFrobenius}. The next result is reminiscent of the Perron-Frobenius theorem, but is instead concerned with the singular vectors and values of a matrix rather than the eigenvectors and eigenvalues.

\begin{lemma}
If every fibre $\bl{T}_{i_1\ldots i_{j-1} \bullet i_{j+1} \ldots i_k }$ is in $\mc{C}(\mc{P}_j)$ then the best rank-one approximation to a non-zero $\bl{T}$ with respect to any of the log-likelihoods mentioned in this section is equal to the best ND rank-one approximation:
\begin{align*}
    \underset{\bs{\theta}: \mathrm{rank}(\bs{\theta}) = 1}{\mathrm{argmax}} \; p(\bl{T}|\bs{\theta}) = \underset{\bs{\theta} \in \mc{N}_{1}}{\mathrm{argmax}} \;\; p(\bl{T}|\bs{\theta}). 
\end{align*}
In particular, if $\bl{T}$ is a matrix with every column in $\mc{C}(\mc{P}_1)$ and every row in $\mc{C}(\mc{P}_2)$ then, up to sign changes, the first left and right singular vectors can be chosen to be in $\mc{C}(\mc{P}_1)$ and $\mc{C}(\mc{P}_2)$ respectively. The largest singular value of $\bl{T}$ has multiplicity one and the best ND rank-one approximation with respect to the Frobenius norm is unique. 
\end{lemma}
\begin{proof}
Throughout this proof it is assumed that the entries of $\bl{T}$ are in the sets outlined in Lemma \ref{lem:ExistenceofDivergenceOptimizer} so that there exists a solution to the optimization problem.

Let $\bs{\theta} = \otimes_{j = 1}^k \bl{v}^{(j)}$ be a rank-one tensor, where we seek vectors $\bl{v}^{(j)} \in \mb{R}^{p_j}$ that maximize the likelihood of interest. 
Under the Gaussian likelihood, upon taking a gradient with respect to $\bl{v}^{(j)}$, the first-order conditions, without any monotonicity constraints, for $\otimes_{j = 1}^k \bl{v}^{(j)}$ to be optimal are
\begin{align}
\label{eqn:GaussianRankOneStationaryEquation}
 \bl{v}^{(j)} =  \frac{1}{\prod_{i \neq j} \Vert \bl{v}^{(i)} \Vert_F^2} \sum_{i_1 = 1}^{p_1} \cdots \sum_{i_{j-1} = 1}^{p_{j-1}} \sum_{i_{j+1} = 1}^{p_{j+1}} \cdots \sum_{i_k = 1}^{p_{k}} \big(\prod_{l \neq j} v^{(l)}_{i_l}\big) \bl{T}_{i_1\ldots i_{j-1} \bullet i_{j+1}\ldots i_k}.
\end{align}
We may assume that any optimal rank-one approximation has $\bl{v}^{(j)} \geq \bl{0}$ as the tensor $\max(\otimes_{j = 1}^k \bl{v}^{(j)}, \bl{0}) = \otimes_{j = 1}^k \max(\bl{v}^{(j)},\bl{0})$, where the maximum is taken elementwise, is closer to $\bl{T}$ than $\otimes_{j = 1}^k \bl{v}^{(j)}$ is when $\otimes_{j = 1}^k \bl{v}^{(j)}$ has a negative entry. The optimal $\bl{v}^{(j)}$ is non-zero when $\bl{T}$ is non-zero since if $T_{i_1\ldots i_k} \neq 0$ then $\bs{\theta} = T_{i_1\ldots i_k} \otimes_{j = 1}^k \bl{e}_{i_j}$ has a larger likelihood value than $\bs{\theta} = \bl{0}$. Equation \eqref{eqn:GaussianRankOneStationaryEquation} shows that $\bl{v}^{(j)}$ must be a conic combination of the monotone fibres $\bl{T}_{i_1\ldots i_{j-1} \bullet i_{j+1}\ldots i_k}$, and so it must itself be in $\mc{C}(\mc{P}_j)$. 

When $k = 2$ and  $\bl{v}^{(1)},\bl{v}^{(2)}$ are any choice of the first left and right singular vectors, by the Eckhart-Young theorem, $\bl{v}^{(1)} \otimes \bl{v}^{(2)}$ is a best rank-one approximation of $\bl{T}$. Therefore, by the above derivation, $\bl{v}^{(j)} \in \mc{C}(\mc{P}_j)$, up to a possible sign change. If the largest singular value was repeated then we have pairs $\bl{v}^{(j)},\bl{w}^{(j)}$ of left ($j = 1$) and right ($j = 2$) singular vectors. As $\mc{C}(\mc{P}_1)$ is proper, the subspace $\text{span}(\bl{v}^{(1)},\bl{w}^{(1)})$ contains a unit norm vector $\bl{z}^{(1)} \notin \pm \mc{C}(\mc{P}_1)$. Defining the $p_j \times 2$ matrix $\bl{R}^{(j)} = [\bl{v}^{(j)}| \bl{w}^{(j)}]$, consider the rotation matrix $\bl{V}$ constructed so that the first column of $\bl{R}^{(1)}\bl{V}$ is equal to $\bl{z}^{(1)}$. Set $\bl{z}^{(2)}$ to be the first column of $\bl{R}^{(2)}\bl{V}$. We have that $\bl{R}^{(1)} (\bl{R}^{(2)})^\intercal = \bl{R}^{(1)} \bl{V} \bl{V}^\intercal (\bl{R}^{(2)})^\intercal$ is a best rank-two approximation of $\bl{T}$ and thus $\bl{z}^{(1)} \otimes \bl{z}^{(2)}$ is a best rank-one approximation of $\bl{T}$. This contradicts the above fact that $\bl{z}^{(1)} \in \pm \mc{C}(\mc{P}_1)$ for any best rank-one approximation. 

Under the multinomial model the unconstrained maximum likelihood estimator (MLE) is well-known \citep[Sec 9.6]{AgrestiCategoricalDataAnalysis} to be equal to the averaged marginal distribution along each mode of the tensor: 
\begin{align*}
    \bl{v}^{(j)} = \frac{1}{\prod_{i \neq j}p_i} \sum_{i_1 = 1}^{p_1} \cdots \sum_{i_{j-1} = 1}^{p_{j-1}} \sum_{i_{j+1} = 1}^{p_{j+1}} \cdots \sum_{i_k = 1}^{p_k} \bl{T}_{i_1\ldots i_{j-1} \bullet i_{j+1}\ldots i_k}.
\end{align*}
It is immediate that $\bl{v}^{(j)} \in \mc{C}(\mc{P}_j)$. The Poisson MLE is equal to $c \otimes_{j = 1}^k \bl{v}^{(j)}$ where each $\bl{v}^{(j)}$ is given as above, and $c = \langle \bl{T},\bl{1} \rangle$ is the sum of all of the entries in $\bl{T}$.

Finally, taking the gradient with respect to $\bl{v}^{(j)}$ for the exponential likelihood yields the equations
\begin{align*}
    \bl{v}^{(j)} = \sum_{i_1 = 1}^{p_1} \cdots \sum_{i_{j-1} = 1}^{p_{j-1}} \sum_{i_{j+1} = 1}^{p_{j+1}} \cdots \sum_{i_k = 1}^{p_k} \frac{ \bl{T}_{i_1\ldots i_{j-1} \bullet i_{j+1}\ldots i_k} }{\prod_{l \neq j} v_{i_l}^{(l)}}.
\end{align*}
This is a conic combination of $\bl{T}_{i_1\ldots i_{j-1} \bullet i_{j+1}\ldots i_k}$ and hence is monotone. Note that none of the entries of any $\bl{v}^{(l)}$ can equal zero at an optimal solution, ensuring that the above first-order condition is well-defined. This is because we define the objective function to be continuous as $\theta_{i_1\ldots i_k}$ approaches $0$, and $\lim_{\theta_{i_1\ldots i_k} \rightarrow 0} \big( \log(\theta_{i_1\ldots i_k}) + \tfrac{T_{i_1\ldots i_k}}{\theta_{i_1 \ldots i_k}}\big) = \infty$.
\end{proof}

We remark that the constraint that each tensor fibre be in $\mc{C}(\mc{P}_j)$ can easily be checked and is weaker than requiring that $\bl{T} \in \mc{N}_{< \infty}$ (Theorem \ref{thm:InjectiveEqualsProjective}).

The following lemma will be applied in the next section to easily find a rank-two factorization. In the case of NMF this observation was used for hierarchical clustering in \citep[Sec 2.3]{gillis2014hierarchicalRankTwo}.

\begin{lemma}
If $\bl{T}$ has an optimal, unconstrained, rank-$r$ decomposition $\bl{T}^{(r)} = \argmin_{\bs{\theta}: \mathrm{rank}(\bs{\theta}) \leq r} \Vert \bl{T} - \bs{\theta} \Vert_F^2$ where $\bl{T}^{(r)} \in \mc{N}_{\leq r}$ then $\bl{T}^{(r)} = \mathrm{argmin}_{\bs{\theta} \in \mc{N}_{\leq r}} \; \Vert \bl{T} - \bs{\theta} \Vert_F^2$. For matrices with $\bl{T}^{(2)} \in \mc{N}_{< \infty}$ we have that $\bl{T}^{(2)} = \mathrm{argmin}_{\bs{\theta} \in \mc{N}_{\leq 2}} \; \Vert \bl{T} - \bs{\theta} \Vert_F^2$. 
\end{lemma}
\begin{proof}
    The first statement is clear from 
    \begin{align*}
    \Vert \bl{T} - \bl{T}^{(r)} \Vert^2_F =        \underset{\bs{\theta}: \mathrm{rank}(\bs{\theta}) \leq r}{\mathrm{min}} \; \Vert \bl{T} - \bs{\theta} \Vert_F^2 \leq \underset{\bs{\theta} \in \mc{N}_{\leq r}}{\mathrm{min}} \;\; \Vert \bl{T} - \bs{\theta}\Vert^2_F.
    \end{align*}
    The second statement follows from Theorem \ref{thm:NDRankTwo}. 
\end{proof}

There are a wide variety of algorithms for computing approximate nonnegative factorizations, the two main classes of which are multiplicative update (MU) algorithms and least squares algorithms; see \cite[Ch 8]{GillisNNMF} for a comprehensive account. Below we propose an algorithm for minimizing the Frobenius norm between a data tensor and a low ND rank tensor. This procedure is closely related to the hierarchical alternating least squares (HALS) algorithm \citep{HALScichocki2007hierarchical} and is based off minimizing the following expression

\begin{align}
\label{eqn:BasicHALSObjectiveFunction}
    \bigg\Vert \bl{T} - \sum_{i \neq s}  \otimes_{j = 1}^k \bl{v}^{(ij)} - \otimes_{j = 1}^k \bl{v}^{(sj)} \bigg\Vert_F^2 \coloneq \bigg\Vert  \tilde{\bl{T}} - \otimes_{j = 1}^k \bl{v}^{(sj)} \bigg\Vert_F^2 
\end{align}
with respect to $\bl{v}^{(st)} \in \mc{C}(\mc{P}_t)$, while holding all other vectors fixed. We have defined $\tilde{\bl{T}} \coloneq \bl{T} - \sum_{i \neq s}  \otimes_{j = 1}^k \bl{v}^{(ij)}$ in \eqref{eqn:BasicHALSObjectiveFunction}. Defining 
\begin{align*}
        \tilde{v}_{l}^{(st)} = \frac{\langle \tilde{\bl{T}}_{\bullet \cdots \bullet l \bullet \cdots \bullet} , \otimes_{j \neq t} \bl{v}^{(sj)} \rangle }{\big\Vert \otimes_{j \neq t} \bl{v}^{(sj)}\Vert_F^2} 
\end{align*}
to be the unconstrained minimizer of $v_l^{(st)}$, the objective function in \eqref{eqn:BasicHALSObjectiveFunction} can be written as 
\begin{align}
\label{eqn:HALsQuadraticObjective}
\sum_{l = 1}^{p_t} \bigg\Vert \tilde{\bl{T}}_{\bullet \cdots \bullet l \bullet \cdots \bullet} - \otimes_{j \neq t} \bl{v}^{(sj)} \tilde{v}_l^{(st)}\bigg\Vert_F^2   +  \bigg(\prod_{j \neq t}\Vert \bl{v}^{(sj)}\Vert_2^2 \bigg)\sum_{l = 1}^{p_t}(\tilde{v}_l^{(st)} - v_l^{(st)})^2,
\end{align}
where only the second term involves $\bl{v}^{(st)}$. This decomposition arises from a sum of squares decomposition in linear regression: $\Vert \bl{y} - \bl{x}\beta \Vert^2_F = \Vert \bl{y} - \bl{x}\hat{\beta} \Vert^2_F + \Vert \bl{x}\hat{\beta} - \bl{x}\beta \Vert^2_F$, with the roles of $\bl{y},\bl{x},\hat{\beta}$, and $\hat{\beta}$ respectively being played by $\tilde{\bl{T}}_{\bullet \cdots \bullet l \bullet \cdots \bullet}, \otimes_{j \neq t} \bl{v}^{(sj)}, \tilde{v}_l^{(st)}$, and $v_l^{(st)}$. Importantly, the second term in \eqref{eqn:HALsQuadraticObjective} is separable in $v_l^{(st)}$, $l \in [p_t]$, indicating that the efficient, pool-adjacent-violators algorithm (PAVA) \citep{PAVAde2010isotone,PAVAyu2016exact} can be used to minimize this quadratic program over $\mc{C}(\mc{P}_t)$. Our ND hierarchical least squares algorithm is a block-coordinate descent routine that proceeds by cycling through the PAVA updates for each vector $\bl{v}^{(ij)}$ until convergence.

\begin{algorithm}
\caption{ND Hierarchical Least Squares}\label{alg:cap}
\begin{algorithmic}
\Require A tensor $\bl{T}$ and a rank $r$. 
\State Initialize $(\bl{v}^{(11)},\ldots,\bl{v}^{(1k)},\bl{v}^{(21)},\ldots,\bl{v}^{(2k)},\ldots\ldots,\bl{v}^{(rk)})$ with $\bl{v}^{(ij)} \in \mc{C}(\mc{P}_j)$. 

\Repeat{
\For{$s = 1,\ldots,r$}
\For{$t = 1,\ldots,k$}
\For{$l = 1,\ldots,p_t$} 
\State $\tilde{\bl{T}}_{\bullet \cdots \bullet l \bullet \cdots \bullet} \gets \bl{T}_{\bullet \cdots \bullet l \bullet \cdots \bullet} - \sum_{i \neq s} v^{(it)}_{l} \otimes_{j \neq  t}^k \bl{v}^{(ij)}$. 
\State $\tilde{v}_l^{(st)} \gets \frac{\langle \tilde{\bl{T}}_{\bullet \cdots \bullet l \bullet \cdots \bullet} , \otimes_{j \neq t} \bl{v}^{(sj)} \rangle }{\big\Vert \otimes_{j \neq t} \bl{v}^{(sj)} \big\Vert_F^2 }$.
\EndFor
\State Update $\bl{v}^{(st)} \gets \argmin_{\bl{v}^{(st)} \in \mc{C}(\mc{P}_t)}  \sum_{l = 1}^{p_t}(\tilde{v}_l^{(st)} - v_l^{(st)})^2$ via PAVA.
\EndFor
\EndFor
}
\Until{The sum $\sum_{i = 1}^r \otimes_{j = 1}^k \bl{v}^{(ij)}$ stabilizes.}
\\
\Return{The rank-$r$ ND approximation $(\bl{v}^{(11)},\ldots,\bl{v}^{(1k)},\bl{v}^{(21)},\ldots,\bl{v}^{(2k)},\ldots\ldots,\bl{v}^{(rk)})$.} 
\end{algorithmic}
\end{algorithm}

Due to the Frobenius norm optimization problem being jointly non-convex in the $\bl{v}^{(ij)}$s there is no guarantee that the sequence of updates converges to a global optimum. However, Proposition 2.7.1 of \citep{BertsekasNonLinearProgramming} ensures that any limit point of the $\bl{v}^{(ij)}$ iterates is a stationary point as long as none of the vectors $\bl{v}^{(ij)}$ are updated to the zero vector over the course of the algorithm. The issue that can occur when any $\bl{v}^{(sj)} = \bl{0}$ is that the componentwise objective function \eqref{eqn:HALsQuadraticObjective} is not strictly convex because the relevant quadratic term vanishes. As there can exist multiple stationary points, it is recommended that multiple, distinct initializations be chosen in the ND HALS algorithm. One possible choice of initialization is to first run an unconstrained, rank-$r$ approximation algorithm on $\bl{T}$, such as alternating least squares \citep{AlternatingLeastSquares}, to obtain vectors $\tilde{\bl{v}}^{(ij)}$. These vectors can then respectively be orthogonally projected onto the cone $\mc{C}(\mc{P}_j)$ by again using PAVA to solve $\bl{v}^{(ij)} = \argmin_{\bl{v}^{(ij)} \in \mc{C}(\mc{P}_j)} \Vert \tilde{\bl{v}}^{(ij)} - \bl{v}^{(ij)} \Vert_2^2$.  

\section{Applications of Low ND Rank Decompositions}
\label{sec:Applications}

In this section two applications of low ND rank factorizations are presented. 
\subsection{Pig Weight Data}
The pig weight dataset \citep[Sec 3.2]  {PigWeightSource} available in the \texttt{fds} R package \citep{fdsRPackageshang2018package} records the weights of $48$ growing pigs over the course of nine weeks. Figure \ref{fig:PigGrowthCurves} illustrates that all of the pigs have monotonically increasing weights over time. If $\bl{T} \in \mb{R}^{48 \times 9}$ is the matrix of weights we may assume that the poset corresponding to the columns is a chain, while no ordering constraints, apart from nonnegativity, are placed on the rows. If $\bl{T} = \sum_{i = 1}^r \bl{a}^{(i)} (\bl{b}^{(i)})^\intercal $ is a rank-$r$ ND factorization, the weight profile of pig $j$ is $\bl{T}_{j \cdot} = \sum_{i = 1}^r a^{(i)}_j (\bl{b}^{(i)})^\intercal$. Up to scale factors, the vectors $\bl{b}^{(i)}$ can be interpreted as growth curves for $r$ different ``subpopulations'' of pigs, while $a^{(i)}_j$ indicates the degree of membership of pig $j$ to subpopulation $i$. The growth curve of each pig is modeled to approximately be a conical combination of the subpopulation curves. If only an NMF was applied to this dataset the $\bl{b}^{(i)}$ may not be as easily interpretable as growth curves. Furthermore, the second, right, singular vector certainly cannot be interpreted as a growth curve as it does not even contain entries that have the same sign.

\begin{figure}[h!]
    \centering
    \begin{subfigure}[t]{0.5\textwidth}
        \centering
        \includegraphics[width = .9\textwidth]{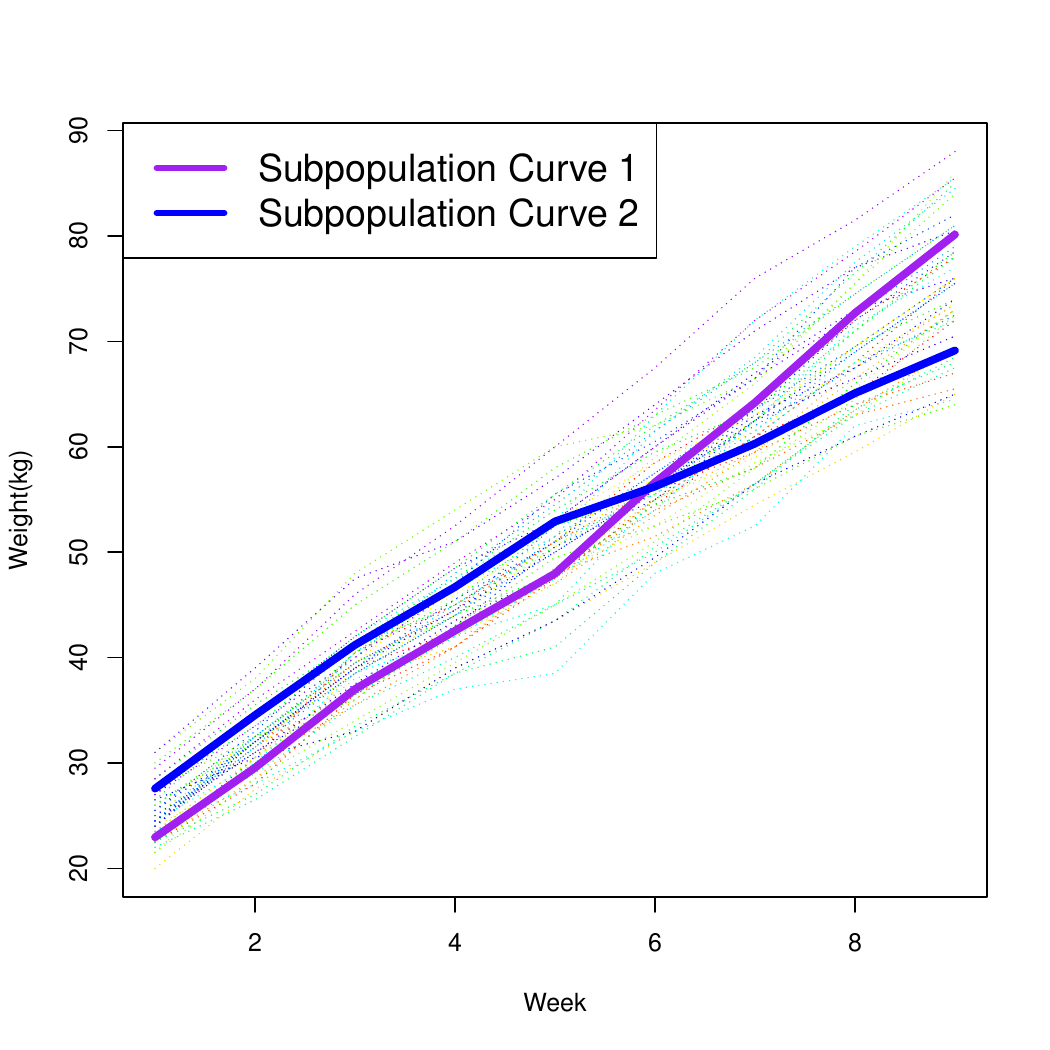}
        \caption{Observed growth curves for each pig, along with the two subpopulation growth curves $\bl{b}^{(1)}$ and $\bl{b}^{(2)}$.}
        \label{fig:PigGrowthCurves}
    \end{subfigure}%
    ~ 
    \begin{subfigure}[t]{0.5\textwidth}
        \centering
        \includegraphics[width = .9\textwidth]{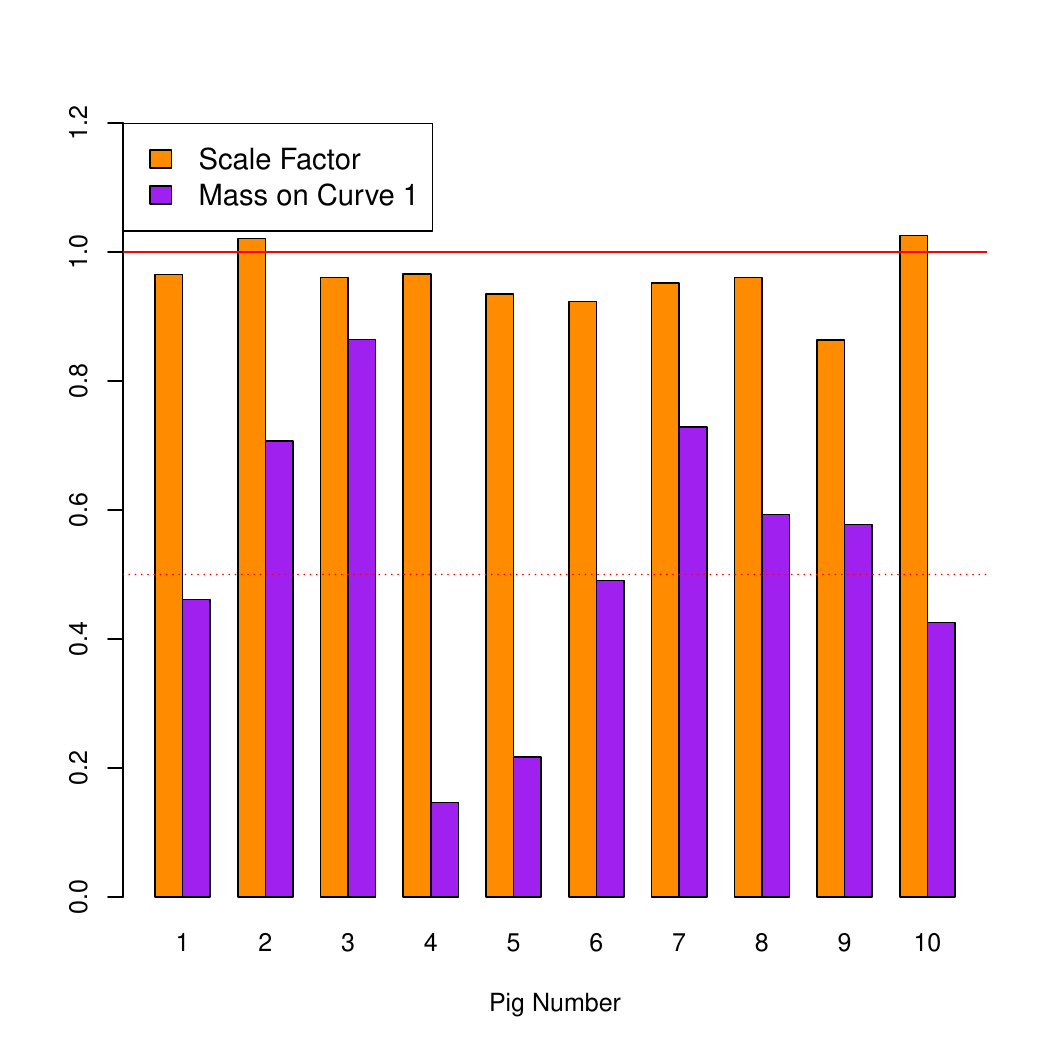}
        \caption{The scale factor $a^{(1)}_i + a^{(2)}_i$ and the relative mass on the first growth curve $a^{(1)}_i/(a^{(1)}_i + a^{(2)}_i)$ for each of the first ten pigs $i = 1,\ldots,10$.}
        \label{fig:PigWeightWeights}
    \end{subfigure}
    \caption{}
        \label{fig:PigWeights}
\end{figure}

To fit this dataset we find an ND rank two approximation $\bl{T} \approx \bl{a}^{(1)} \otimes \bl{b}^{(1)} + \bl{a}^{(2)} \otimes \bl{b}^{(2)}$ with $\bl{a}^{(i)} \in \mb{R}^{48}_+$ and $\bl{b}^{(i)} \in \mc{C}(\{1,\ldots,9\})$. A rank of two was chosen because this rank is sufficient to yield a low-reconstruction error for $\bl{T}$; the first four squared singular values of $\bl{T}$ are  $121{,}000$, $613$, $129$, and $109$. The rank-two SVD approximation $\bl{T}^{(2)}$ of $\bl{T}$ has nonnegative entries and monotone increasing rows, implying that $\bl{T}^{(2)} \in \mc{N}_{\leq 2}$ (Theorem \ref{thm:NDRankTwo}). The vectors $\bl{b}_1,\bl{b}_2$ can be taken to be the pair of rows of $\bl{T}^{(2)}$ that have the largest angle between each other. Such a pair can easily be found by computing dot products between the rows of $\bl{T}^{(2)}$ that are rescaled to have unit norm. This approach has been used in the NMF literature for matrices that have a nonnegative rank of two which have the property of being separable \citep[Sec 4.1]{GillisNNMF}. The corresponding coefficients $\bl{a}^{(1)},\bl{a}^{(2)}$ are found by a linear regression. Lastly, we rescale $\bl{b}^{(1)},\bl{b}^{(2)}$ so that the mean weight of each $\bl{b}^{(i)}$ is equal to the average weight across all the pigs: $\tfrac{1}{9}\sum_{l = 1}^9 b^{(i)}_l = \tfrac{1}{9*48}\sum_{i = 1}^{48} \sum_{j = 1}^9 T_{ij}$. Figure \ref{fig:PigGrowthCurves} displays the two estimated growth curves $\bl{b}^{(1)}$ and $\bl{b}^{(2)}$, overlaid on the observed growth curves. Interestingly, around week five the slope of both estimated curves change, a phenomena that may warrant further study. In Figure \ref{fig:PigWeightWeights} we see that the fitted curve $a_i^{(1)}\bl{b}^{(1)} + a_i^{(2)} \bl{b}^{(2)}$, for pig $i$, is nearly a convex combination of $\bl{b}^{(1)}$ and $\bl{b}^{(2)}$, as the scale factor $a^{(1)}_i + a^{(2)}_i$ is close to one. The relative importance of $\bl{b}^{(1)}$ as compared to $\bl{b}^{(2)}$ noticeably varies across the pigs. 

One issue that has not yet been discussed in this work is that ND factorizations are not unique in general, even after accounting for scaling and permutation of the rank-one factors. Various regularization penalties or additional constraints have been introduced in the NMF literature \citep[Ch 4]{GillisNNMF} to select a factorization that has desirable properties. In the above analysis we have implicitly selected the minimum volume \citep{MinVolumeNMFmiao2007endmember} ND factorization by choosing rows from $\bl{T}^{(2)}$ to serve as the $\bl{b}^{(i)}$. An alternative here is take a maximum volume factorization by choosing the $\bl{b}^{(i)}$ to be the extremal rays of the cone $\mc{C}(\{1,\ldots,9\}) \cap \text{row}(\bl{T}^{(2)})$. Either factorization will yield the same reconstruction error $\Vert \bl{T} - \bl{T}^{(2)} \Vert^2_F$, but will lead to potentially different interpretations with respect to the $\bl{b}^{(i)}$ curves and the $\bl{a}^{(i)}$ membership weights.


\subsection{Perceived Mental Health of Canadians}
We examine a subset of the data from the annual Canadian Community Health Survey (CCHS) \citep{StatCanData} from $2016$-$2022$ that records responses of whether an individual perceives their mental health to be either fair or poor. The survey also records the sex and age group of the participant. The $(5,7,2)$-dimensional tensor $\bl{T}$ in Table \ref{tab:MentalHealthData} displays the proportion of respondents who indicated that they had either poor or fair mental health.

\begin{table}[ht]
\centering
\begin{tabular}{|c||rrrrrrr|}
  \hline
 Age Groups & \multicolumn{7}{|c|}{Female} 
  \\
  \hline
 & 2016 & 2017 & 2018 & 2019 & 2020 &  2021 & 2022   \\ 
  \hline
  \hline
12-17 & 6.0&	7.8&	8.5&	8.4&	12.9&	16.5&	21.0  \\ 
  18-34 &  9.0&	10.1&	12.1&	13.1&	15.3&	17.8&	24.2\\ 
  35-49 & 7.6&	7.9&	8.4&	8.2&	10.8&	14.6&	16.4  \\ 
  50-64 & 8.3&	8.1&	7.9&	7.5&	9.3&	11.6&	14.0 \\ 
  65+ & 5.7&	4.7&	5.4&	5.5&	6.0&	6.8&	9.2 \\ 
  \hline
& \multicolumn{7}{|c|}{Male} 
\\
\hline
  12-17 & 2.9&	3.9&	5.0&	3.8&	3.8&	7.5&	8.7  \\ 
  18-34 & 6.7&	6.6&	7.6&	10.6&	11.2&	13.6&	16.1 \\ 
  35-49 & 5.2&	6.2&	6.8&	7.3&	9.9&	11.6&	13.5 \\ 
  50-64 & 	7.3&	6.4&	7.6&	6.8&	8.7&	9.0&	11.7 \\ 
  65+ & 6.3&	6.0&	4.5&	4.9&	4.9&	6.7&	7.9  \\ 
  \hline
   \hline
\end{tabular}
\caption{Percentage of respondents who indicated that they perceived their mental health to be either fair or poor.}
\label{tab:MentalHealthData}
\end{table}

Using the ND HALs algorithm from Section \ref{sec:Optimization} we fit both rank-two and rank-one approximations to the data tensor of the form $\sum_{i = 1}^r \bl{a}^{(i)} \otimes \bl{b}^{(i)} \otimes\bl{c}^{(i)}$. The residual sum of squares for the respective approximations are $145$ and $55$ with the total sum of squares equaling $6925$. We only impose nonnegativity constraints on the sex mode, while the Hasse diagrams in Figure \ref{fig:MentalHealthHasse} provides orderings on the age group and years that are consistent with the observed data. Specifically, the mental health of respondents deteriorated in 2020 and the subsequent two years, conceivably as a result of the COVID-19 pandemic. Respondents in the second youngest age group were most likely to perceive their mental health as fair or poor while respondents in the oldest age group were most likely to not provide this response.

There is a marked difference between how individuals in the 12-17 age group responded between females and males, with females having poorer perceived mental health. In fitting a rank-two ND decomposition the two rank-one factors provide a low-dimensional summary of this sex-age interaction. The rank-one factors of the two rank-one terms are shown in Table \ref{tab:MentalHealthRankOneFactors}. These factors are scaled so that $\Vert \bl{c}^{(i)}\Vert_1 = 1$ and $\Vert \bl{b}^{(i)} \Vert_1 = 7$. We see that the factor $\bl{a}^{(2)} \otimes \bl{b}^{(2)} \otimes \bl{c}^{(2)}$ only alters the female portion of the tensor and represents a deviation between the male and female responses. In comparison with $\bl{a}^{(1)}$, the 12-17 age group in $\bl{a}^{(2)}$ takes on larger values relative to the 18-34 age group. Furthermore, the impact of pandemic is more apparent in $\bl{b}^{(2)}$ as compared to $\bl{b}^{(1)}$. In summary, the low ND rank decomposition allows for a concise interpretation of the data tensor and illustrates that young females in 2022 were particularly prone to respond that they had fair or poor mental health as compared to young males.

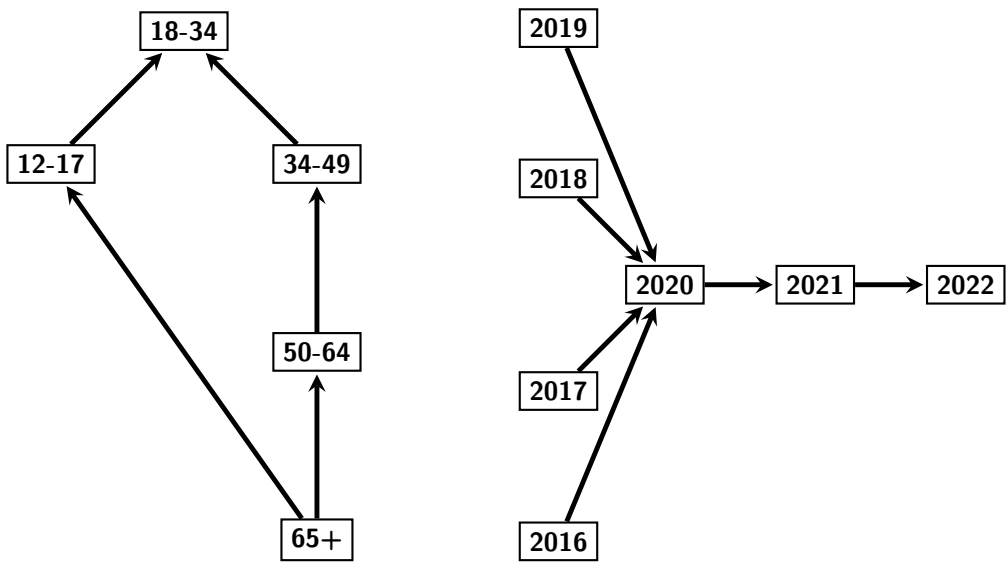
\begin{figure*}[!ht]
    \centering
    \begin{subfigure}[t]{0.5\textwidth}
        \centering
\begin{tikzpicture}[->,>=stealth,shorten >=1pt,auto,node distance=2.5cm,thick,main node/.style={draw,font=\sffamily\small\bfseries}]
  \node[main node] (3) {18-34};
  \node[main node] (1) [below left of=3] {12-17};
  \node[main node] (2) [below right of=3] {34-49};
    \node[main node] (4) [below of=2] {50-64};
      \node[main node] (5) [below of=4] {65+};
  \path[every node/.style={font=\sffamily\small},line width=.65mm]
    (1) edge [left] node[left] {} (3)
    (5) edge (4)
    (5) edge (1)
    (4) edge (2)
    (2) edge [right] node[right] {} (3);
\end{tikzpicture}
    \end{subfigure}%
    ~ 
    \begin{subfigure}[t]{0.4\textwidth}
        \centering
\begin{tikzpicture}[->,>=stealth,shorten >=1pt,auto,node distance=2cm,thick,main node/.style={draw,font=\sffamily\small\bfseries}]
  \node[main node] (20) {2020};
  \node[main node] (21) [right of=20] {2021};
    \node[main node] (22) [right of=21] {2022};
    \node[main node] [above left of=20]  (18) {2018};
    \node[main node] [below left of=20] (17) {2017};
      \node[main node] [below of=17] (16) {2016};
            \node[main node] (19)[above of=18] {2019};
  \path[every node/.style={font=\sffamily\small},line width=.65mm]
    (19) edge  (20)
    (20) edge [right] (21)
        (21) edge [right] (22)
        (18) edge (20)
          (17) edge (20)
            (16) edge (20);
\end{tikzpicture}
    \end{subfigure}
    \caption{The respective Hasse diagrams indicating orderings over the age groups and years.}
    \label{fig:MentalHealthHasse}
\end{figure*}

\begin{table}[h!]
    \centering
    \begin{subtable}[t]{0.3\textwidth}
        \centering
        \begin{tabular}{|c|c|c|}
            \hline
          Age group  & $\bl{a}^{(1)}$  & $\bl{a}^{(2)}$ \\
            \hline
           12-17 & 8.93  & 8.28 \\
           18-34 &  16.77 & 8.28 \\
           35-49 & 14.23 & 5.06 \\
           50-64 & 13.18 & 4.17 \\
           65+ & 8.93 & 2.53 \\
            \hline
        \end{tabular}
        \subcaption{Age group factors.}
        \label{tab:sub_a}   
    \end{subtable}
    \hfill 
    \begin{subtable}[t]{0.3\textwidth}
        \centering
        \begin{tabular}{|c|c|c|}
            \hline
          Year  & $\bl{b}^{(1)}$  & $\bl{b}^{(2)}$ \\
            \hline
           2016 & 0.76  & 0.57 \\
           2017 &  0.75 &  0.70\\
           2018 & 0.82 & 0.78\\
               2019 & 0.90 & 0.75 \\
                   2020 & 1.02 & 1.07\\
                       2021 & 1.25 & 1.34\\
                           2022 & 1.49 & 1.78\\
            \hline
        \end{tabular}
        \subcaption{Year factors.}
        \label{tab:sub_b}
    \end{subtable}
        \begin{subtable}[t]{0.3\textwidth}
        \centering
        \begin{tabular}{|c|c|c|}
            \hline
    Sex     & $\bl{c}^{(1)}$  & $\bl{c}^{(2)}$ \\
            \hline
           Female & 0.38  & 1 \\
           Male &  0.62 & 0 \\
            \hline
        \end{tabular}
        \subcaption{Sex factors.}
        \label{tab:sub_b}
    \end{subtable}
        \caption{Factors appearing in the rank-two approximation $\bl{T} \approx \bl{a}^{(1)} \otimes \bl{b}^{(1)} \otimes \bl{c}^{(1)} +  \bl{a}^{(2)} \otimes \bl{b}^{(2)} \otimes \bl{c}^{(2)}$. }
    \label{tab:MentalHealthRankOneFactors}
\end{table}

\section{Conclusion and Future Directions}
In this article the concept of the nondecreasing rank was introduced and properties of nondecreasing factorizations were examined. There remains a wealth of relevant, open questions to be explored, a few of which we mention here. One property of ND factorizations not extensively discussed is the uniqueness of such factorizations. Like with nonnegative matrix factorizations, it is expected that only under similarly strong conditions, such as separability \citep[Sec 4.2.2]{GillisNNMF}, will a matrix have a unique ND factorization. Also not discussed, is how to choose a suitable rank $r$ for performing an ND factorization. An easy method for doing this is to plot the value of a divergence objective function against $r$ and look for a kink in the plot where the divergence stops decreasing rapidly in response to an increase of $r$ \citep{ElbowPlotchoi2017selecting}. In Section \ref{sec:Optimization} an algorithm is introduced for minimizing the Frobenius norm between a tensor and its ND approximation. Developing efficient multiplicative update \citep{MultiplicativeUpdateleesueng2000algorithms} algorithms for finding ND approximations would be a valuable contribution. To construct such an algorithm more constraints as compared to NMF, stemming from the monotonicity constraints of the order cones, have to be contended with.

Nondecreasing factorizations are closely related to ideas in the field of order constrained statistical inference \citep{OrderRestrictedBarlowBartholomew,OrderRestrictedRobertsonDykstra}. Beyond monotonicity constraints, related shape constraints for functions, such as convexity and unimodality, would be interesting to inspect for low-rank structures. Unimodal matrix factorizations have previously been examined in \citep{ang2021nonnegativeUnimodal,bro1998UnimodalPositivity}. Unimodality is related to the umbrella order where $x_1 \prec \cdots \prec x_{l-1} \prec x_l \succ x_{l+1} \cdots \succ x_p$. The main distinction is that in a unimodal vector it is not known which of the entries is the maximum, while for the umbrella order above the $x_l$th entry is posited to be the maximum. The non-convex cone of unimodal vectors is therefore equal to a union of umbrella-order cones. Problems involving convex, unimodal, and monotone factorizations can also be contextualized from an infinite-dimensional, functional data analysis perspective, where notions of convex sets and rank continue to be well-defined.

\section*{Acknowledgments}
The author acknowledges the support of the Natural Sciences and Engineering Research Council of Canada (NSERC), [RGPIN-2025-03968, DGECR-2025-00237].

Cette recherche a été financée par le Conseil de recherches en sciences naturelles et en génie du Canada (CRSNG), [RGPIN-2025-03968, DGECR-2025-00237].

\bibliographystyle{abbrv}
\bibliography{biblio}

\end{document}